%% file: paper.tex
\documentclass[sigconf]{acmart}

\AtBeginDocument{%
  \providecommand\BibTeX{{%
    \normalfont B\kern-0.5em{\scshape i\kern-0.25em b}\kern-0.8em\TeX}}}

\copyrightyear{2021}
\acmYear{2021}
\setcopyright{acmcopyright}\acmConference[AIES '21]{Proceedings of the 2021 AAAI/ACM Conference on AI, Ethics, and Society}{May 19--21, 2021}{Virtual Event, USA}
\acmBooktitle{Proceedings of the 2021 AAAI/ACM Conference on AI, Ethics, and Society (AIES '21), May 19--21, 2021, Virtual Event, USA}
\acmPrice{15.00}
\acmDOI{10.1145/3461702.3462517}
\acmISBN{978-1-4503-8473-5/21/05}

\usepackage{xspace}
\usepackage{pifont}
\usepackage{booktabs}
\usepackage{diagbox}
\usepackage{bbm}
\usepackage{todonotes}
\usepackage{paralist} 
\usepackage{enumitem}

\usepackage{color}

\usepackage{amsthm}
\usepackage{amsmath,amssymb,amsfonts}
\usepackage{algorithmic}
\usepackage{graphicx}
\usepackage{textcomp}
\usepackage{url}
\usepackage[algo2e,boxed]{algorithm2e}
\usepackage[subrefformat=parens,labelformat=parens]{subfig}
\usepackage{tabularx}

\input{dfn}

\begin{document}
	\fancyhead{}
\title{\method: Fairness-aware Outlier Detection}

\author{Shubhranshu Shekhar}
\email{shubhras@andrew.cmu.edu}
\affiliation{%
 \institution{Carnegie Mellon University}
 \city{Pittsburgh}
 \state{PA}
  \country{USA}
}

\author{Neil Shah}
\email{nshah@snap.com}
\affiliation{%
  \institution{Snap Inc.}
  \city{Seattle}
  \state{WA}
  \country{USA}}

\author{Leman Akoglu}
\email{lakoglu@andrew.cmu.edu}
\affiliation{%
   \institution{Carnegie Mellon University}
  \city{Pittsburgh}
  \state{PA}
  \country{USA}
}

\renewcommand{\shortauthors}{Shekhar, Shah and Akoglu}

\begin{abstract}

\input{000abstract}
\end{abstract}

\ccsdesc[300]{Computing methodologies~Machine learning algorithms}
\ccsdesc[500]{Computing methodologies~Anomaly detection}

\keywords{fair outlier detection;
	outlier detection;
	anomaly detection;
	algorithmic fairness;
	end-to-end detector;
	deep learning}

\maketitle

\section{Introduction}
\label{sec:introduction}
\input{010introduction-short}

\section{Desiderata for Fair Outlier Detection}
\label{sec:prelim}
\input{020preliminary}

\section{Fairness-aware Outlier Detection}
\label{sec:method}
\input{030method}

\section{Experiments}
\label{sec:exp}

\input{040experiment}

\section{Related Work}
\label{sec:related}
\input{050related}
\section{Conclusions}
\label{sec:conclusion}
\input{060conclusion}

\begin{acks}
		This research is sponsored by NSF CAREER 1452425. In addition, we thank Dimitris Berberidis for helping with the early development of the ideas and the preliminary code base. Conclusions expressed in this material are those of the authors and do not necessarily reflect the views, expressed or implied, of the funding parties.
\end{acks}

\bibliographystyle{ACM-Reference-Format}
\bibliography{paper}

\pagebreak
\pagebreak
\newpage
\appendix
\label{sec:appendix}
\input{appendix}

\end{document}

%% file: dfn.tex
\newtheorem{problem}{Problem}

\newtheorem{claim}{Claim}
\newtheorem{informal}{Informal Problem}
\newcommand{\hide}[1]{}

\newcommand{\bit}{\begin{compactitem}}
\newcommand{\eit}{\end{compactitem}}
\newcommand{\ben}{\begin{compactenum}}
\newcommand{\een}{\end{compactenum}}

\newcommand{\beq}{\begin{equation}}
\newcommand{\eeq}{\end{equation}}
\newcommand{\indep}{\perp \!\!\! \perp}

\newcounter{x}

\newcommand{\pv}{PV}
\newcommand{\flag}{O} 
\newcommand{\lbl}{Y}

\newcommand{\samples}{N}
\newcommand{\nfeats}{d}
\newcommand{\xobs}{X}
\newcommand{\R}{\mathbb{R}} 
 
\newcommand{\prob}{P}
\newcommand{\s}{s}

\newcommand{\ndcg}{{\sc NDCG}\xspace}

\newcommand{\prem}{{preprocessing}\xspace}
\newcommand{\inm}{{in-processing}\xspace}

\newcommand{\base}{{\sc base}\xspace}
\newcommand{\corr}{{\sc FairOD-L}\xspace} 
\newcommand{\corrscore}{{\sc FairOD-C}\xspace} 
\newcommand{\method}{{\sc FairOD}\xspace}
\newcommand{\disparate}{{\sc dir}\xspace}
\newcommand{\lfr}{{\sc lfr}\xspace}
\newcommand{\rw}{{\sc rw}\xspace}
\newcommand{\crl}{{\sc arl}\xspace}

\newcommand{\loss}{\mathcal{L}}

\newcommand{\lossC}{\mathcal{L}_{SP}}
\newcommand{\lossF}{\mathcal{L}_{GF}}

\newcommand{\enc}{G_E}
\newcommand{\dec}{G_D}
\newcommand{\aez}{Z}
\newcommand{\recon}{G}

\newcommand{\numdatasets}{{4\xspace}}

\newcommand{\synthone}{{\texttt{Synth1}}\xspace}
\newcommand{\synthtwo}{\texttt{Synth2}\xspace}
\newcommand{\adult}{\texttt{Adult}\xspace}
\newcommand{\credit}{\texttt{Credit}\xspace}
\newcommand{\tweets}{\texttt{Tweets}\xspace}
\newcommand{\ads}{\texttt{Ads}\xspace}

\newcommand{\fairness}{{\small \sffamily Fairness}\xspace}
\newcommand{\gf}{{\small \sffamily GroupFidelity}\xspace}
\newcommand{\topkrank}{{\small \sffamily Top-$k$ Rank Agreement}\xspace}

\newcommand{\apratio}{{\small \sffamily AP-ratio}\xspace}
\newcommand{\prratio}{{\small \sffamily AUC-ratio}\xspace}

\newcommand{\auc}{{\small \sffamily AUC}\xspace}
\newcommand{\ap}{{\small \sffamily AP}\xspace}
\newcommand{\fr}{{\small \sffamily Flag-rate}\xspace}

\long\def\ignore#1{}

\newsavebox\CBox

\newcommand{\bluecolor}[1]{{\textsf{\textcolor{blue!90}{#1}}}}

%% file: 000abstract.tex
Fairness and Outlier Detection (OD) are closely related, as it is exactly the goal of OD to spot rare, minority samples in a given population. However, when being a minority (as defined by protected variables, such as race/ethnicity/sex/age) does not reflect positive-class membership (such as criminal/fraud),  OD produces unjust outcomes. Surprisingly, fairness-aware OD has been almost untouched in prior work, as fair machine learning literature mainly focuses on supervised settings.
Our work aims to bridge this gap.  Specifically, we develop desiderata  
capturing well-motivated fairness criteria for OD,
and systematically formalize the fair OD problem. Further, guided by our desiderata, we propose \method, a fairness-aware outlier detector that has the following desirable properties: \method (1) 
exhibits treatment parity
at test time, (2) aims to flag equal proportions of samples from all groups (i.e. obtain group fairness, via statistical parity), and (3) strives to flag truly high-risk 
samples within each group. 
Extensive experiments on a diverse set of synthetic and real world datasets show that \method produces outcomes that are fair with respect to protected variables, while performing comparable to (and in some cases, even better than) fairness-agnostic detectors in terms of detection performance. 

%% file: 010introduction-short.tex

Fairness in machine learning (ML) has received a surge of attention in the recent years. The community has largely focused on designing different notions of fairness \cite{barocas2017fairness, journals/corr/abs-1808-00023, verma2018fairness} mainly tailored towards supervised ML problems~\cite{hardt2016equality, zafar2017fairness, goel2018non}.  However, perhaps surprisingly, fairness in the context of outlier detection~(OD) is vastly understudied. OD is critical for numerous applications in security~\cite{gogoi2011survey, zavrak2020anomaly, zhang2006anomaly}, finance~\cite{van2015apate,lee2020autoaudit,johnson2019medicare}, healthcare~\cite{luo2010unsupervised, bosc2003automatic} etc. and is widely used for detection of rare positive-class instances. 

\textbf{Outlier detection for ``policing''}:~ 
In such critical systems, OD is often used to flag instances that reflect \emph{riskiness}, 
which are then ``{policed}'' (or audited) by human experts. For example, law enforcement agencies might employ automated surveillance systems in public spaces to spot suspicious individuals based on visual characteristics, who are subsequently stopped and frisked. 
Alternatively, in the financial domain, analysts can police fraudulent-looking claims, and corporate trust and safety employees can police bad actors on social networks.
\begin{figure*}[t!]
	\centering
	{\includegraphics[height=1.42in]{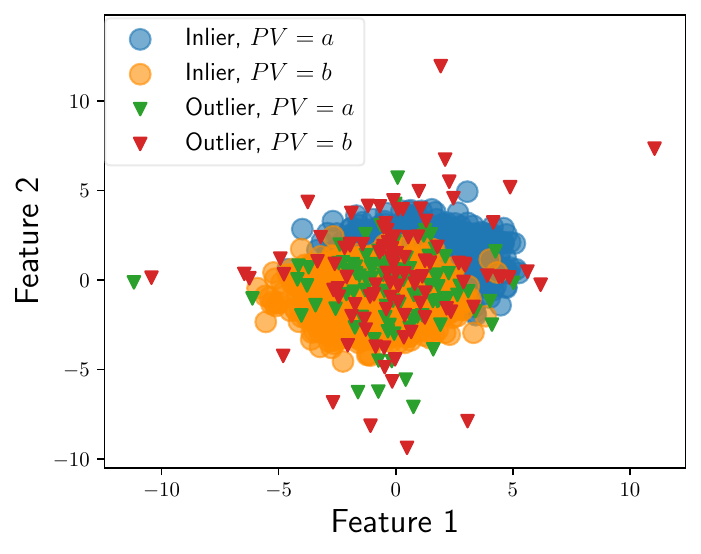}}
	{\includegraphics[height=1.42in]{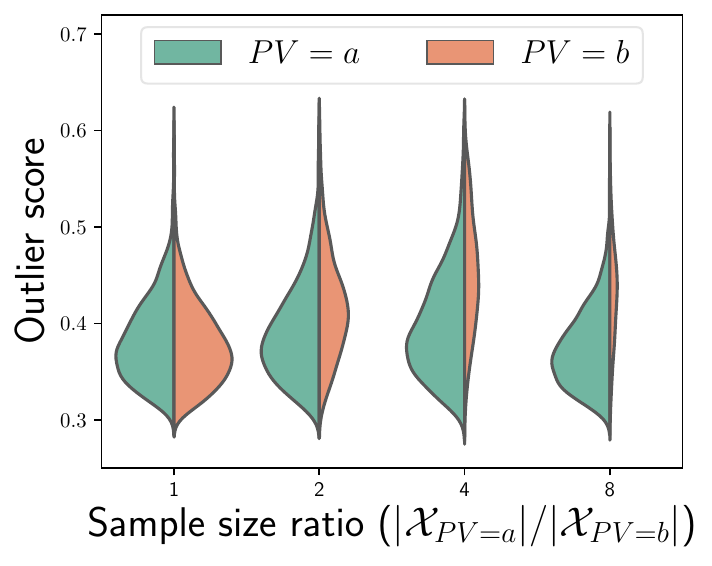}}
	{\includegraphics[height=1.42in]{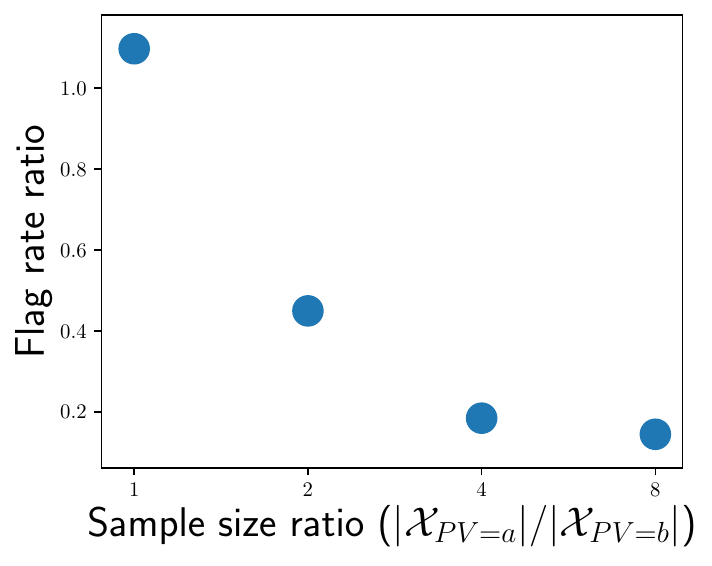}}
	\caption{ (left) Simulated 2-dim. data with equal sized groups i.e. $|\mathcal{X}_{PV=a}|{=}|\mathcal{X}_{PV=b}|$. (middle) Group score distributions induced by $\pv=a$ and $\pv=b$ are plotted by varying the simulated $|\mathcal{X}_{PV=a}|/|\mathcal{X}_{PV=b}|$ ratio. Notice that minority group ($\pv=b$) receives larger outlier scores as the size ratio increases. (right) Flag rate ratio of the groups for the varying sample size ratio $|\mathcal{X}_{PV=a}|/|\mathcal{X}_{PV=b}|$. As we increase size disparity, the minority group is ``policed'' (i.e. flagged) comparatively more.\label{fig:disparity}}
\end{figure*}

\textbf{Group sample size disparity yields unfair OD}:~ Importantly, outlier detectors are designed exactly to spot rare, \emph{statistical minority} samples\footnote{In this work, the words sample, instance, and observation are used interchangeably throughout text.} with the hope that outlierness reflects riskiness, which prompts their bias against \textit{societal minorities} (as defined by race/ethnicity/sex/age/etc.) as well, since minority group sample size is by definition small. 

However, when minority status (e.g. Hispanic) does not reflect positive-class membership (e.g. fraud), OD produces {\bfseries{\em unjust outcomes, by overly flagging the instances from the minority gro-\\ups as outliers.}} This conflation of statistical and societal minorities can become an ethical matter. 


\textbf{Unfair OD leads to disparate impact}:~ 
What would happen downstream if we did not strive for \textit{fairness-aware} OD given the existence of societal minorities?
OD models' inability to distinguish societal minorities (as induced by so-called \textit{protected} variables ($\pv$s)), from statistical minorities, contributes to the likelihood of minority group members being flagged as outliers~(see Fig. \ref{fig:disparity}). This is further exacerbated by proxy variables which partially-redundantly encode (i.e. correlate with) the $\pv$(s)\ignore{PV(s)}, by increasing the number of subspaces in which minorities stand out.
The result is \emph{overpolicing} due to over-representation of minorities in OD outcomes.  Note that overpolicing the minority group also implies underpolicing the majority group given limited policing capacity and constraints. 

Overpolicing can also feed \textit{back} into a system when the policed outliers are used as labels in downstream supervised tasks. Alarmingly, this initially skewed sample (due to unfair OD), may be amplified through a feedback loop via predicting policing where more outliers are identified in more heavily policed groups.  Given that OD's use in societal applications has direct bearing on social well-being, ensuring that OD-based outcomes are non-discriminatory is pivotal.  This demands the design of fairness-aware OD models, which our work aims to address. 

\textbf{Prior research and challenges}:~ Abundant work on algorithm fairness has focused on supervised ML tasks~\cite{beutel2019putting, hardt2016equality, zafar2017fairness}. Numerous notions of fairness ~\cite{barocas2017fairness, verma2018fairness} have been explored in such contexts, each with their own challenges in achieving equitable decisions~\cite{journals/corr/abs-1808-00023}. In contrast, there is little to no work on addressing fairness in \textit{unsupervised} OD. Incorporating fairness into OD is challenging, in the face of (1) many possibly-incompatible notions of fairness and, (2) the absence of ground-truth 
outlier labels. 

The two works tackling\footnote{\cite{conf/ecai/DavidsonR20} aims to \textit{quantify} fairness of OD model outcomes \textit{post hoc}, which thus has a different scope.} unfairness in the OD literature are by P and Abraham~\cite{p2020fairOD} which proposes an ad-hoc procedure to introduce fairness specifically to the LOF algorithm~\cite{breunig2000lof}, and Zhang and Davidson~\cite{zhang2020towards} (concurrent to our work) which proposes an adversarial training based deep SVDD detector. Amongst other issues (see Sec. \ref{sec:related}), the approach proposed in \cite{p2020fairOD} invites disparate treatment, necessitating explicit use of $\pv$ \textit{at decision time}, leading to taste-based discrimination~\cite{corbett2018measure} that is unlawful in several critical applications. On the other hand, the approach in \cite{zhang2020towards} has several drawbacks (see Sec. \ref{sec:related}), and in light of unavailable implementation, we include a similar baseline called ~\crl{} that we compare against our proposed method.

Alternatively, one could re-purpose existing fair representation learning techniques~\cite{zemel2013learning, edwards2015censoring, beutel2017data} as well as data preprocessing strategies~\cite{kamiran2012data, feldman2015certifying} for subsequent fair OD.
However, as we show in Sec. \ref{sec:exp} and discuss in Sec. \ref{sec:related}, isolating representation learning from the detection task is suboptimal, largely (needlessly) sacrificing detection performance for fairness. 

\begin{figure}[t!]
	\centering
	\includegraphics[width=\columnwidth]{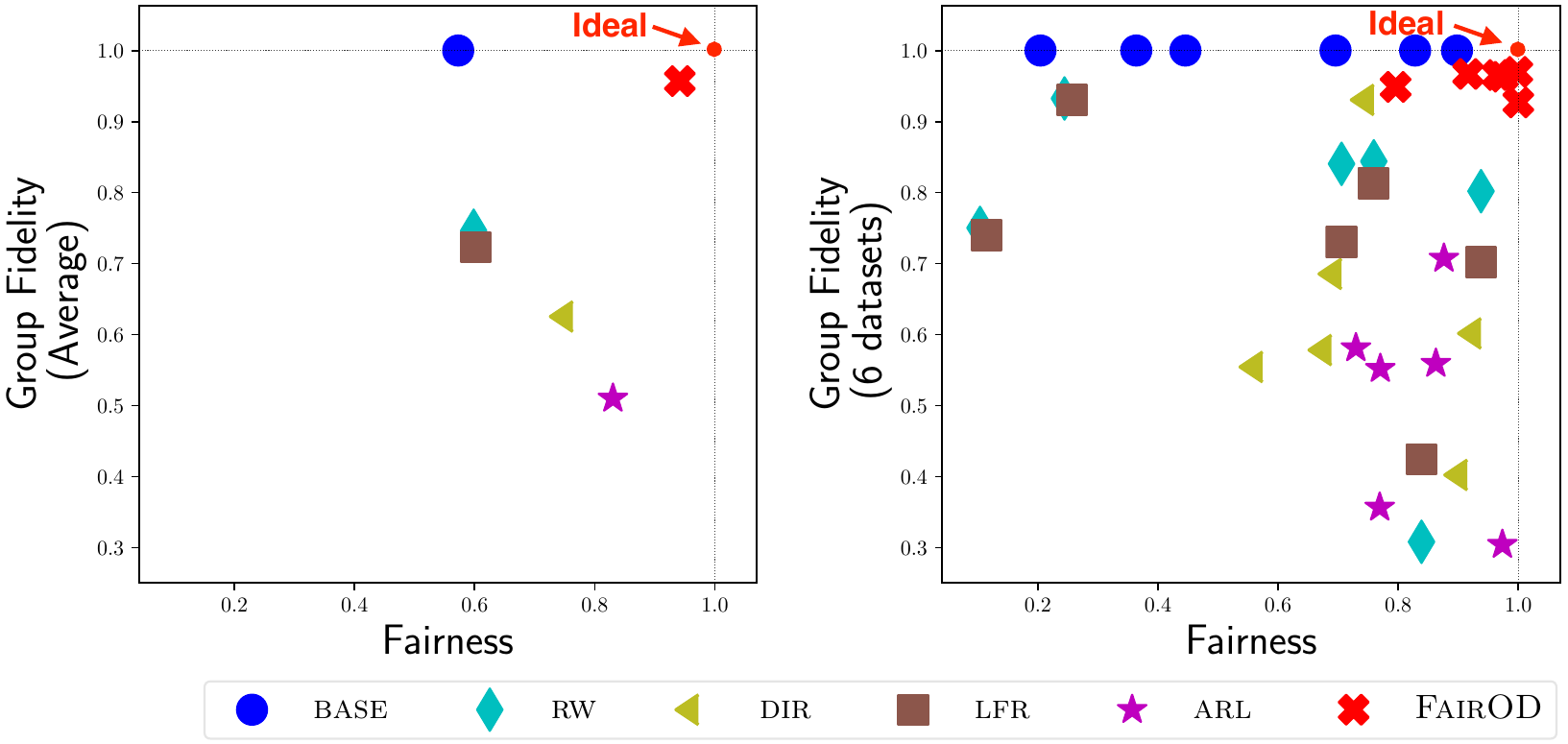}
	\caption{\fairness (statistical parity) vs. \gf (group-level rank preservation) of baselines and our proposed \method (red cross), (left) averaged across 6 datasets, and (right) on individual datasets. \method outperforms existing solutions (tending towards ideal), achieving fairness while preserving group fidelity from the \base detector. See Sec. \ref{sec:exp} for more details.\label{fig:combinedrank}}
\end{figure}

\textbf{Our contributions}:~
Our work strives to design a fairness-aware OD model to achieve equitable policing across groups and avoid an unjust conflation of statistical and societal minorities.  We summarize our main contributions as follows: 

\begin{compactenum}
	\item{\bf Desiderata \& Problem Definition for Fair Outlier Detection:}  We identify 5 properties characterizing detection quality and fairness in OD as desiderata for fairness-aware detectors. 
	We discuss their justifiability and achievability, based on which we formally define the (unsupervised)  fairness-aware OD problem (Sec. \ref{sec:prelim}).

		\item {\bf Fairness Criteria \& New, 
		Fairness-Aware OD Model: } We introduce well-motivated fairness criteria and give mathematical objectives that can be optimized to obey the desiderata. These criteria are universal, in that they can be embedded into the objective of any end-to-end OD model. We propose \method, a new detector which directly incorporates the prescribed criteria into its training. 
		Notably, \method (1) aims to equalize flag rates across groups, achieving group fairness via statistical parity, while (2) striving to flag truly high-risk samples within each group, and (3) avoiding
		disparate treatment.
		(Sec. \ref{sec:method})
		
	\item {\bf Effectiveness on Real-world Data: } We  apply  \method on several real-world and synthetic datasets with diverse applications such as credit risk assessment and hate speech detection. Experiments demonstrate \method's effectiveness in achieving both fairness goals  (Fig.~\ref{fig:combinedrank}) as well as accurate detection~(Fig.~\ref{fig:accuracy}, Sec. \ref{sec:exp}), significantly outperforming alternative solutions. 

\end{compactenum}

{\bf Reproducibility:} All of our source code and  datasets are shared publicly at \texttt{\url{https://tinyurl.com/fairOD}}. 

%% file: 020preliminary.tex
\subsubsection*{Notation}
 We are given $\samples$ samples (also, observations or instances) 
$\mathcal{X} = \{\xobs_i\}_{i=1}^\samples  \subseteq \mathbb{R}^\nfeats$ 
as the input for OD where $\xobs_i \in \R^\nfeats$ denotes the feature representation for observation $i$. Each observation is additionally associated with a binary\footnote{For simplicity of presentation, we consider a single, binary protected variable (PV). We discuss extensions to multi-valued PV and multi-attribute PVs in Sec. \ref{sec:method}.} protected (also, sensitive) variable, 
$\mathcal{PV} = \{\pv_i\}_{i=1}^\samples$, 
where $\pv_i \in \{a, b\}$ identifies two groups -- the majority ($\pv_i = a$) group and the minority ($\pv_i = b$) group.
We use $\mathcal{Y}=\{Y_i\}_{i=1}^\samples$,  $\lbl_i \in \{0, 1\}$, to denote the \emph{unobserved} ground-truth binary labels for the observations where, for exposition, $\lbl_i=1$ denotes an outlier (positive outcome) and $\lbl_i=0$ denotes an inlier (negative outcome).
We use $\flag :  {X} \mapsto \{0,1\}$ to denote the predicted outcome of an outlier detector, and $\s: {X} \mapsto \mathbb{R}$ to capture the corresponding numerical outlier score  as the estimate of the outlierness.  Thus, $\flag(X_i), \s(X_i)$ respectively indicate predicted outlier label and outlier score for sample $X_i$. We use $\mathcal{O} = \{ \flag(X_i) \}_{i=1}^\samples$ and $\mathcal{S} = \{ \s(X_i) \}_{i=1}^\samples$ to denote the set of all predicted labels and scores from a given model without loss of generality.  Note that we can derive $\flag(X_i)$ from a simple thresholding of $\s(X_i)$.  We routinely drop $i$-subscripts to refer to properties of a single sample without loss of generality.
We denote the group \textit{base rate} (or {prevalence}) of outlierness  as $br_a=\prob(Y=1|PV=a)$ for the majority group. Finally, we let $fr_a=\prob(O=1|PV=a)$ depict the \textit{flag rate} of the detector for the majority group. Similar definitions extend to the minority group with $\pv=b$. 
Table~\ref{tab:symb} gives a list of the notations frequently used  throughout the paper.

Having presented the problem setup and notation, we state our fair OD problem (informally) as follows.

\begin{informal}[\bluecolor{Fair Outlier Detection}]
	Given samples $\mathcal{X}$ and protected variable values $\mathcal{PV}$,
	estimate outlier scores $\mathcal{S}$ and assign outlier labels $\mathcal{O}$, 
	such that
	\bit
		\item[(i)] assigned labels and scores are ``fair'' w.r.t. the $\pv$, and
		\item[(ii)] higher scores correspond to higher riskiness encoded by the underlying (unobserved) $\mathcal{\lbl}$. 
	\eit
\end{informal}

How can we design a fairness-aware OD model that is \emph{not biased} against minority groups? What constitutes a ``fair'' outcome in OD, that is, what would characterize fairness-aware OD? What specific notions of fairness are most applicable to OD?

To approach the problem and address these motivating questions, we first propose a list of desired properties that an ideal fairness-aware detector should satisfy, and whether, in practice, the desired properties can be enforced 
followed by 
our proposed solution, \method.

\begin{table}
	\caption{Frequently used symbols and definitions. \label{tab:symb}}
\resizebox{\columnwidth}{!}{
	\centering
	\begin{tabular}{ll}
		\toprule
		Symbol & Definition \\
		\midrule
		$\xobs$ & $d$-dimensional feature representation of an observation\\
		$\lbl$ & true label of an observation, w/ values 0 (inlier), 1 (outlier)\\ 
		$\pv$ &  binary protected (or sensitive) variable, w/ groups $a$ (majority), $b$ (minority) \\ 
		$\flag$ & detector-assigned label to an observation, w/ value 1 (predicted/flagged outlier)\\
		$br_v$ & base rate of/fraction of ground-truth outliers in group $v$, i.e. $br_v=\prob(Y=1|PV=v)$\\
		$fr_v$ & flag rate of/fraction of flagged observations in group $v$, i.e $fr_v=\prob(O=1|PV=v)$\\
		\bottomrule
	\end{tabular}
}
\end{table}

\subsection{Proposed Desiderata}

\label{subsec:desi}
	
	\noindent {D1. \bfseries Detection effectiveness: } We require an OD model to be accurate at detection, such that the scores assigned to the instances by OD are well-correlated with the ground-truth outlier labels. Specifically, OD benefits the policing 
	effort only when the detection rate (also, precision) is strictly larger than the \emph{base rate} (also, prevalence), that is, 
	\beq 
	\label{eq:ppv}
	\prob(\lbl = 1\; |\; \flag = 1) > \prob(\lbl =1) \;.
	\eeq
	This condition ensures that any policing effort concerted through the employment of an OD model is able to achieve a \textit{strictly larger  precision} (LHS) \textit{as compared to random sampling}, where policing via the latter would simply yield a precision that is equal to the prevalence of outliers in the population (RHS) in expectation. Note that our first condition in \eqref{eq:ppv} 
	is related to detection performance, and specifically, the usefulness of OD itself for policing applications. 
	
	 {\em How-to: }  We can indirectly control for detection effectiveness via careful feature engineering. Assuming domain experts assist in feature design,  it would be reasonable to expect a better-than-random detector that satisfies Eq.~\eqref{eq:ppv}.

	Next, we present \emph{fairness-related} conditions for OD.

\vspace{0.075in}
	\noindent {D2. \bfseries Treatment parity: } OD should exhibit non-disparate treatment that explicitly avoid the use of $\pv$ for producing a decision. In particular, OD decisions should obey \beq
		\prob(\flag = 1 \;|\; \xobs)
	\;=\; 
	\prob(\flag = 1\;|\; \xobs, \pv = v), \; \forall v \;. 
	\eeq
	In words, the probability that the detector outputs an outlier label $\flag$ for a given feature vector $\xobs$ remains unchanged even upon observing the value of the $\pv$.  In many settings (e.g. employment), explicit $\pv$ use is unlawful at inference.

	  {\em How-to: }  We can build an OD model using a disparate learning process~\cite{lipton2018does} that uses $\pv$ only during the model training phase, but does not require access to $\pv$ for producing a decision, hence satisfying treatment parity.

	Treatment parity ensures that OD decisions 
	are effectively ``blindfolded'' to the $\pv$.  However, this notion of fairness alone is not sufficient to ensure equitable policing across groups; namely, removing the $\pv$ from scope may still allow discriminatory OD results for the minority group (e.g., African American) due to the presence of several other features (e.g., zipcode) that (partially-)redundantly encode the $\pv$.  Consequently, by default, OD will use the $\pv$ \textit{indirectly}, through access to those correlated proxy features.  Therefore, additional conditions follow.

	\vspace{0.075in}
	\noindent {D3. \bfseries Statistical parity (SP): } 
	One would expect the OD outcomes to be independent of group membership, i.e. $O \indep PV$.
	In the context of OD, this notion of fairness (also, demographic parity, group fairness, or independence) aims to enforce that the outlier flag rates are independent of $\pv$ and equal across the groups as induced by $\pv$. 
	
	Formally, an OD model satisfies statistical parity under a distribution over $(\xobs, \pv)$ where $\pv \in \{a, b\}$ if 
	%
		\beq
		\begin{aligned}
			\label{eq:sp}
		fr_a = fr_b \;\; \text{or equivalently,} \;\;\; \\
		\prob(\flag = 1 |\pv = a) \;=\; 
		\prob(\flag = 1 | \pv = b) \;.
		\end{aligned}
		\eeq
	SP implies that the fraction of minority (majority) members in the flagged set is the same as the fraction of minority (majority) in the overall population. Equivalently, one can show 
\begin{align}
\label{eq:pop}
    fr_a = fr_b \; &(\text{SP})  \iff    \prob(PV=a|O=1)=\prob(PV=a) \; \nonumber \\ 
    & \text{and}\; \prob(PV=b|O=1)=\prob(PV=b) \;.
\end{align}

	The motivation for SP derives from luck egalitarianism~\cite{09luck} -- a family of egalitarian theories of distributive justice that aim to counteract the distributive effects of ``brute luck''. 
	By redistributing equality to those who suffer through no fault of their own choosing, mediated via race, gender, etc.,
	it aims to counterbalance the manifestations of such ``luck''.  Correspondingly, SP ensures equal flag rates across $\pv$ groups, eliminating such group-membership bias. Therefore, it merits incorporation in OD since OD results are used for policing or auditing by human experts in downstream applications.\looseness=-1

		 {\em How-to: } 
	 We could enforce SP during OD model learning by comparing the distributions of the predicted outlier labels $\flag$ amongst groups, and update the model to ensure that these output distributions match across groups. 

	SP, however, is not sufficient to ensure both equitable \textit{and} accurate outcomes as it
	permits so-called ``laziness''~\cite{barocas2017fairness}.
    Being an unsupervised quantity that is agnostic to the ground-truth labels $\mathcal{Y}$, SP could be satisfied while producing decisions that are arbitrarily inaccurate for any or all of the groups.
    In fact, an extreme scenario would be random sampling; where we select a certain fraction of the given population uniformly at random and flag all the sampled instances as outliers. As evident via Eq. \eqref{eq:pop}, this entirely random procedure would achieve SP (!).
    The outcomes could be worse -- that is, not only inaccurate (put differently, as accurate as random) but also unfair for only \textit{some} group(s) -- when OD flags mostly the true outliers from one group while flagging randomly selected instances from the other group(s), leading to discrimination \textit{despite} SP.
     Therefore, additional criteria is required to explicitly penalize ``laziness,'' aiming to not only flag \emph{equal fractions} of instances across groups but also those \textit{true outlier} instances 
     from both groups.
    

	\vspace{0.075in}
	\noindent {D4. \bfseries Group fidelity (also, Equality of Opportunity):} It is desirable that the \textit{true} outliers are equally likely to be assigned higher scores, and in turn flagged, regardless of their membership to any group as induced by $\pv$. We refer to this notion of fairness as group fidelity, which steers OD outcomes toward being faithful to the ground-truth outlier labels equally across groups, obeying the following condition
	\beq 
	\label{eq:fidelity}
	\prob(\flag = 1 |  \lbl = 1, \pv = a ) \;=\; \prob(\flag = 1 |  \lbl = 1, \pv = b )\;.
	\eeq
	Mathematically, this condition is equivalent to the so-called Equality of Opportunity\footnote{Opportunity, because positive-class assignment by a supervised model in many fair ML problems  is often associated with a positive outcome, such as being hired or approved a loan.} in the supervised fair ML literature, and is a special case of Separation~\cite{verma2018fairness, hardt2016equality}. In either case, it requires that all $\pv$-induced groups experience the same true positive rate. 
	Consequently, it penalizes ``laziness'' by ensuring that the true-outlier instances are ranked above (i.e., receive higher outlier scores than) the inliers within each group.\looseness=-1
	
	The key caveat here is that \eqref{eq:fidelity} is a supervised quantity that requires access to the ground-truth labels $\mathcal{Y}$, which are explicitly unavailable for the \textit{unsupervised} OD task. 
	What is more, various impossibility results have shown that certain fairness criteria, including SP and Separation, are mutually exclusive or incompatible~\cite{barocas2017fairness},
	implying that simultaneously satisfying both of these conditions (exactly) is not possible.
	
	  {\em How-to: }
	The unsupervised OD task does not have access to $\mathcal{Y}$, therefore, group fidelity cannot be enforced directly.
	Instead, 
	we propose to enforce group-level rank preservation 
	that maintains fidelity to within-group ranking from the \base model, where \base is a fairness-agnostic OD model. Our intuition is that rank preservation acts as a proxy for group fidelity, or more broadly Separation, via our assumption that within-group ranking in the \base model is accurate and top-ranked instances within each group encode the highest risk samples within each group. 
	
	Specifically, let $\pi^{\text{\base}}$ represent the ranking of instances based on \base OD scores, and let $\pi^{\text{\base}}_{\pv = a}$ and $\pi^{\text{\base}}_{\pv = b}$ denote the group-level ranked lists for majority and minority groups, respectively. Then, the rank preservation is satisfied when $\pi^{\text{\base}}_{\pv = v} = \pi^{}_{\pv = v}; \forall v \in \{a, b\}$ where $\pi^{}_{\pv = v}$ is the ranking of group-$v$ instances based on outlier scores from our proposed OD model. Group rank preservation aims to address the ``laziness'' issue that can manifest while ensuring SP; we aim to not lose the within-group detection prowess of the original detector while maintaining fairness. Moreover, since we are using only a proxy for Separation, the mutual exclusiveness of SP and Separation may no longer hold, though we have not established this mathematically.
	
	\vspace{0.075in}
	\noindent {D5. \bfseries Base rate preservation: } The flagged outliers from OD results are often audited and then  used as human-labeled data for supervised detection (as discussed in previous section) which can introduce bias through a feedback loop. Therefore, it is desirable that  group-level base rates within the flagged population 
	is reflective of the group-level base rates in the overall population, so as to not introduce group bias of outlier incidence downstream. In particular, we expect OD outcomes to ideally obey
	\begin{align}
	&\prob(\lbl =1|\flag=1, \pv =a) = br_a\;, \text{  and  } \\ 
	&\prob(\lbl =1|\flag=1, \pv =b) = br_b \;. 
	\end{align}
Note that group-level base rate within the flagged population (LHS) is mathematically equivalent to group-level precision in OD outcomes, and as such, is also a supervised quantity which suffers the same caveat as in D4, regarding unavailability of $\mathcal{Y}$.

 {\em How-to: } As noted, $\mathcal{Y}$ is not available to an unsupervised OD task. Importantly, provided an OD model satisfies D1 and D3, we show that it cannot simultaneously also satisfy D5, i.e. per-group equal base rate in OD results (flagged observations) and in the overall population.

\begin{claim}
	\label{claim1}
	Detection effectiveness: $\prob(\lbl = 1 | \flag = 1) > \prob(\lbl =1)$ and SP:
	$\prob(\flag = 1 | \pv = a) = \prob(\flag = 1 | \pv = b)$ jointly imply that $\prob(\lbl =1|\flag=1,\pv =v) > \prob(\lbl =1| \pv =v), \exists v.$ 
\end{claim}
\begin{proof}
	We prove the claim in Appendix\footnote{\texttt{\url{https://tinyurl.com/fairOD}}} ~\ref{app:c1}.
\end{proof}


Claim \ref{claim1} shows an incompatibility and states that, provided D1 and D3 are satisfied, the base rate in the flagged  population cannot be equal to (but rather, is an overestimate of) that in the overall population for \textit{at least one of the groups}. As such, base rates in OD outcomes cannot be reflective of their true values.
Instead, one may hope for the preservation of the \textit{ratio} of the base rates (i.e. it is not impossible). As such,
a relaxed notion of D5 is to preserve proportional base rates across groups in the OD results, that is,
\begin{equation}
\label{eq:ratio}
\frac{\prob(\lbl =1|\flag=1,\pv =a)}{\prob(\lbl =1|\flag=1,\pv =b)} = \frac{\prob(\lbl =1| \pv =a)}{\prob(\lbl =1|\pv =b)} \;.
\end{equation}
Note that ratio preservation still cannot be explicitly enforced as \eqref{eq:ratio} is also label-dependent. 
Finally we show in Claim \ref{claim2} that, provided D1, D3 and Eq. \eqref{eq:ratio} are all satisfied, 
then it entails that 
the base rate in OD outcomes is an overestimation of the true group-level base rate \emph{for every group}.

\begin{claim}
	\label{claim2}
	Detection effectiveness: $\prob(\lbl = 1 | \flag = 1) > \prob(\lbl =1)$, SP: $\prob(\flag = 1 | \pv = a)=\prob(\flag = 1 | \pv = b)$, and Eq. \eqref{eq:ratio}: $\frac{\prob(\lbl =1|\flag=1,\pv =a)}{\prob(\lbl =1|\flag=1,\pv =b)} = \frac{\prob(\lbl =1| \pv =a)}{\prob(\lbl =1|\pv =b)}$ jointly imply $\prob(\lbl=1 | \pv=v, \flag=1) $$>$$ \prob(\lbl=1 | \pv=v), \forall v$.
\end{claim}
\begin{proof}
	We prove the claim in Appendix ~\ref{app:c2}.
\end{proof}

Claim~\ref{claim1} and Claim~\ref{claim2} indicate that if we have both ($i$) better-than-random precision (D1) and ($ii$) SP (D3), interpreting the base rates in OD outcomes for  downstream learning tasks would not be meaningful, as they would not be reflective of true population base rates.
Due to both these incompatibility results, and also feasibility issues given the lack of $\mathcal{Y}$, we leave base rate preservation -- despite it being a desirable property -- out of consideration.

\subsection{Problem Definition}
\label{sub:problem}
Based on the definitions and enforceable desiderata, our fairness-aware OD problem is formally defined as follows:


\begin{problem}[\bluecolor{Fairness-Aware Outlier Detection}]
	\label{prob:fair}
Given samples $\mathcal{X}$ and protected variable values $\mathcal{PV}$, estimate outlier scores $\mathcal{S}$ and assign outlier labels $\mathcal{O}$, to achieve 
		\bit
		\item[(i)] $\prob(\lbl = 1 | \flag = 1) > \prob(\lbl =1)$\;, 
		
		\hspace*{\fill} [Detection effectiveness] 
		\item[(ii)] $\prob(\flag \;|\; \xobs, \pv = v) = \prob(\flag \;|\; \xobs), \; \forall v \in \{a, b\}$\;, 
		
		\hspace*{\fill} [Treatment parity] 
		\item[(iii)]  $\prob(\flag = 1 | \pv = a) = \prob(\flag = 1 | \pv = b)$\;,  
		
		\hspace*{\fill} [Statistical parity]
		\item[(iv)]  $\pi^{\text{\base}}_{\pv = v} = \pi^{}_{\pv = v} \;,  \forall v \in \{a, b\}$, where \base is a fairness-agnostic detector.  \hspace*{\fill} [Group fidelity proxy]\looseness=-1
		\eit
\end{problem}


Given a dataset 
along with $\pv$ values, the goal is to design an OD model that builds on an existing \base OD model and satisfies the criteria $(i)$--$(iv)$, following the proposed 
desiderata D1 -- D4.

\subsection{Caveats of a Simple Approach}
A simple yet na\"ive fairness-aware OD approach to address Problem~\ref{prob:fair} 
can be designed as follows:
\begin{enumerate}
	\item Obtain ranked lists $\pi^{\text{\base}}_{\pv = a}$ and $\pi^{\text{\base}}_{\pv = b}$ from \base, and
	\item Flag top instances as outliers from each ranked list at equal fraction such that \\ $\prob(\flag = 1 | \pv = a) = \prob(\flag = 1 | \pv = b), \pv \in \{a, b\}$
\end{enumerate}
This approach fully satisfies $(iii)$ and $(iv)$ in Problem~\ref{prob:fair} by design, as well as $(i)$ given suitable features.  However, 
it explicitly suffers from \textit{disparate treatment}.


%% file: 030method.tex
In this section, we describe our proposed \method{} -- an unsupervised, fairness-aware, end-to-end OD model that embeds our proposed learnable (i.e. optimizable) fairness constraints \ignore{which can be plugged} into an existing \base OD model.
The key features of our model are that \method aims for equal flag rates across groups (statistical parity), and encourages correct top group ranking (group fidelity), while not requiring $\pv$ for decision-making on new samples (non-disparate treatment). As such, it aims to target the proposed desiderata D1 -- D4 as described in Sec. \ref{sec:prelim}. 


\subsection{Base Framework}
\label{subsec:method}

Our proposed OD model instantiates a deep-autoencoder (AE)  framework for the base outlier detection task. However, we remark that the fairness regularization criteria introduced by \method can be plugged into any end-to-end \textit{optimizable} anomaly detector, such as one-class support vector machines~\cite{scholkopf2001estimating}, deep anomaly detector~\cite{chalapathy2018anomaly}, variational AE for OD~\cite{an2015variational}, and deep one-class classifiers~\cite{pmlr-v80-ruff18a}. Our choice of AE as the \base OD model stems from the fact that AE-inspired methods have been shown to be state-of-the-art outlier detectors~\cite{chen2017outlier, ma2013parallel, zhou2017anomaly} 
and that our fairness-aware loss criteria can be optimized in conjunction with the objectives of such models. The main goal of \method is to incorporate our proposed notions of fairness into an end-to-end OD model, irrespective of the choice of the \base model family. 

AE consists of two main components: an encoder $\enc: \xobs \in \mathbb{R}^d \mapsto \aez \in \mathbb{R}^m$ and  a decoder $\dec: \aez \in \mathbb{R}^m \mapsto X \in \mathbb{R}^d$. $\enc(\xobs)$ encodes the input $\xobs$ to a hidden vector (also, code) $\aez$ that preserves the important aspects of the input. Then, $\dec(\aez)$ aims to generate $\xobs^\prime$, a reconstruction of the input from the hidden vector $\aez$.  Overall, the AE can be written as $\recon = \dec \circ \enc$, such that $\recon(X) = \dec\left(\enc\left(X\right)\right)$.
For a given AE based framework, the outlier score for $\xobs$ is computed using the reconstruction error as
\begin{align}
\label{eq:score}
	\s(\xobs) = \| \xobs - \recon(\xobs)\|_2^2 \;.
\end{align}

Outliers tend to exhibit large reconstruction errors because they do not conform to to the  patterns in the data as coded by an auto-encoder, hence the use of  reconstruction errors as outlier scores~\cite{aggarwal2015outlier, pang2020deep, shah2014spotting}.
This scoring function is general in that it applies to many reconstruction-based OD models, which have different parameterizations of the reconstruction function $\recon$. 
We show in the following how \method regularizes the reconstruction loss from \base through fairness constraints that are conjointly optimized during the training process. The \base OD model optimizes the following
\begin{equation}
\label{eq:reconstruction}
\mathcal{L}_{\text{\base}}= \sum_{i=1}^{\samples}  \| \xobs_i - \recon(\xobs_i)\|_2^2
\end{equation}
and we denote its outlier scoring function as $\s^{\text{\base}}(\cdot)$.


\subsection{Fairness-aware Loss Function}
\label{subsec:criteria}

We begin with designing a loss function for our OD model that optimizes for achieving SP and group fidelity by introducing regularization to the \base objective criterion.
Specifically, \method minimizes the following loss:
\begin{align}
	\label{eq:loss}
	\resizebox{0.875\hsize}{!}{%
	\boxed{
	\loss \;\; = \;\; \alpha \underbrace{\mathcal{L}_{\text{\base}}}_{\text{Reconstruction}} \;\;+ \;\; (1-\alpha) \underbrace{\lossC}_{\text{Statistical Parity}} \;\;+ \;\; \gamma \underbrace{\lossF}_{\text{Group Fidelity}}
	}
	}
\end{align}
where $\alpha \in (0, 1)$ and $\gamma > 0$ are  hyperparameters which govern the balance between different fairness criteria and reconstruction quality in the loss function.

The first term in Eq.~\eqref{eq:loss} is the objective for learning the reconstruction (based on \base model family)  as given in Eq.~\eqref{eq:reconstruction}, which
  quantifies the {goodness} of the encoding $\aez$ via the squared error between the original input and its reconstruction generated from $\aez$.
The second component in Eq.~\eqref{eq:loss} corresponds to regularization introduced to enforce the fairness notion of independence, or statistical parity (SP) as given in Eq.~\eqref{eq:pop}. Specifically, the term seeks to minimize the absolute correlation between the outlier scores $\mathcal{S}$ (used for producing predicted labels $\mathcal{O}$) and protected variable values $\mathcal{PV}$. $\lossC$ is given as 
\begin{align}
	\label{eq:corrloss}
	\lossC = \left\lvert \frac{\big(\sum_{i=1}^{\samples} \s(\xobs_i) - \mu_{\s} \big)\; \big(\sum_{i=1}^{\samples} \pv_i - \mu_{\pv} \big)}    {\sigma_{\s} \; \sigma_{\pv}} \right\rvert
\end{align}
where $\mu_{\s} = \frac{1}{\samples} \sum_{i=1}^{\samples} \s(\xobs_i)$, $\sigma_{\s} =  \frac{1}{\samples}\sum_{i=1}^{\samples} (\s(\xobs_i) - \mu_{\s})^2$, $\mu_{\pv} = \frac{1}{\samples} \sum_{i=1}^{\samples} \pv_i$, and $\sigma_{\pv} =  \frac{1}{\samples}\sum_{i=1}^{\samples} (\pv_i - \mu_{\pv})^2$.

We adapt this absolute correlation loss from \cite{beutel2019putting}, which proposed its use in a supervised setting with the goal of enforcing statistical parity.  As \cite{beutel2019putting} mentions, while minimizing this loss does not guarantee independence, it performs empirically quite well and offers stable training. We observe the same in practice; it leads to minimal associations between OD outcomes and the protected variable (see details in Sec. \ref{sec:exp}).

Finally, the third component of  Eq.~\eqref{eq:loss} emphasizes that \method should maintain fidelity to within-group rankings from the \base model (penalizing ``laziness'').
We set up a listwise learning-to-rank objective in order to enforce group fidelity. Our goal is to train \method such that it reflects the within-group rankings based on  $\s^{\text{\base}}(\cdot)$ from \base. To that end, we employ a listwise ranking loss criterion that is based on the well-known   Discounted Cumulative Gain (DCG) \cite{jarvelin2002cumulated} measure, often used to assess ranking quality in information retrieval tasks such as search. 
For a given ranked list, DCG is defined as
\begin{equation*}
\text{DCG} = \sum_r \frac{2^{rel_r}-1}{\log_2(1 +r)} \label{equ:DCG_stronger}
\end{equation*}
where 
$rel_r$ depicts the relevance of the item 
ranked at the $r^{th}$ position. 
In our setting, we use the outlier score $\s^{\text{\base}}(X)$ of an instance $X$ to reflect its relevance since we aim to mimic the group-level ranking by \base. As such, DCG per group can be re-written as

\noindent
\resizebox{\linewidth}{!}{
	\begin{minipage}{\linewidth}
 \begin{align*} 
 \text{DCG}_{\pv=v}
 &= \sum_{\xobs_i \in \mathcal{X}_{\pv=v}} \frac{2^{{\s^{\text{\base}}(X_i)}}-1}
 {\log_2\big( 1+ \sum_{\xobs_k \in \mathcal{X}_{\pv=v}} \mathbbm{1}{[{\s(X_i)} \leq {\s(X_k)}}] \big)}
 \end{align*}
 \end{minipage}
}
where 
$\mathcal{X}_{\pv=a}$ and $\mathcal{X}_{\pv=b}$ would respectively denote the set of observations from majority and minority groups, and 
$\s(X)$ is the estimated outlier score from our \method model under training. 

A key challenge with DCG is that it is not differentiable, as it involves ranking (sorting).
Specifically, the sum term in the denominator 
uses the (non-smooth) indicator function $\mathbbm{1}(\cdot)$ to obtain the position of instance $i$ as ranked by the estimated outlier scores. 
We circumvent this challenge by replacing the indicator function by the (smooth) sigmoid approximation, following  \cite{qin2010general}.
Then, the group fidelity loss component $\lossF$ is given as
\beq
	\label{eq:rankloss}	
	\begin{aligned}
		\lossF  = \sum_{v \in \{a,b\}} \left(1 -  \sum_{\xobs_i \in \mathcal{X}_{\pv=v}}\frac{ 2^{\s^{\text{\base}}(\xobs_i)} -1 }
		{\text{\textsc{dnm}}} \right)
	\end{aligned}
\eeq
\resizebox{.96\linewidth}{!}{
	\begin{minipage}{\linewidth}
\begin{align*}
\text{\textsc{dnm}} = \log_2\big( 1 + \sum_{\xobs_k \in \mathcal{X}_{\pv=v}} \text{sigm}( \s(\xobs_k) - \s(\xobs_i)  ) \big) \cdot IDCG_{\pv=v}, 
\end{align*}
\end{minipage}
}

 $\text{sigm}(x) =$ $\frac{\exp(-cx)}{1+ \exp(-cx)}$ is the sigmoid function where $c>0$ is the scaling constant, and,
$ IDCG_{\pv=v} = \sum_{j=1}^{|\mathcal{X}_{\pv=v}|}({ (2^{\s^{\text{\base}}(\xobs_j)} -1)}$ $/{\log_2(1+j)})$
 is the ideal (hence $I$), i.e. largest  DCG value attainable for the respective group. Note that IDCG can be computed per group apriori to model training via \base outlier scores alone, and serves as a normalizing constant in Eq. \eqref{eq:rankloss}.

 Note that having trained our model, scoring instances does not require access to the value of their $\pv$, as $\pv$ is only used in Eq.  \eqref{eq:corrloss} and \eqref{eq:rankloss} for training purposes. At test time, the anomaly score of a given instance $X$ is computed simply via Eq. \eqref{eq:score}. Thus, \method also fulfills the desiderata on treatment parity.

\begin{table*}[!ht]
	\caption{Summary statistics of real-world and synthetic datasets used for evaluation.\label{tab:datatable}}
	\vspace{-.1in}
	\centering{\resizebox{\textwidth}{!}{
		\begin{tabular}{lrrrrccr}
			\toprule
			{\bf Dataset}  & {$\mathbf{\samples}$} & $\mathbf{d}$ & $\mathbf{\pv}$  & $\mathbf{\pv=b}$ & $\mathbf{|\mathcal{X}_{PV=a}|/|\mathcal{X}_{PV=b}|}$ & \textbf{$\%$ outliers} & \textbf{Labels}\\
			\midrule
			\adult      & 25262  & 11  &  {gender} &  {\em female} & 4 & 5 & \{{income} $\leq 50K$,  {income} $> 50K$\}\\
			\credit  & 24593   & 1549  &  {age} &  \textit{age} $\leq 25$& 4 & 5 & \{{paid,  delinquent}\}\\
			\tweets      & 3982  & 10000  & {racial dialect} &  \textit{African-American} & 4 & 5 & \{{normal,  abusive}\}\\
			\ads       & 1682   & 1558   & {simulated} & $1$ & 4 & 5 & \{{non-ad,  ad}\}\\
			\midrule
			\synthone & 2400& 2 & {simulated}& $1$ & 4 & 5 & $\{0, 1\}$\\
			\synthtwo & 2400& 2 & {simulated}& $1$ & 4 & 5 &$\{0, 1\}$\\
			\bottomrule
		\end{tabular}
		}
	}
\end{table*}

\subsubsection*{Optimization and Hyperparameter Tuning} We optimize the parameters of \method by minimizing the loss function given in Eq.~\eqref{eq:loss} by using the built-in Adam optimizer~\cite{kingma2014adam} implemented in PyTorch.

\method comes with two tunable hyperparameters, $\alpha$ and $\gamma$. We define a grid for these and pick the configuration that achieves the best balance between SP and our proxy  quantity for group fidelity (based on group-level ranking preservation). Note that both of these quantities are unsupervised (i.e., do not require access to ground-truth labels), therefore, \method model selection can be done in a completely unsupervised fashion.
We provide further details about hyperparameter selection in Sec. \ref{sec:exp}.

\subsubsection*{Generalizing to Multi-valued and Multiple Protected Attributes}\hfill
\label{subsec:generalization}
\noindent\emph{Multi-valued $\pv$}.$\;$ \method generalizes beyond binary $\pv$, and easily applies to settings with multi-valued, specifically categorical $\pv$ such as race. Recall that $\lossC$ and $\lossF$  are the loss components that depend on $\pv$.
For a categorical $\pv$, $\lossF$ in Eq. ~\eqref{eq:rankloss} would simply remain the same, where the outer sum goes over all unique values of the $\pv$.
For $\lossC$, one could one-hot-encode (OHE) the $\pv$ into multiple variables and minimize the correlation of outlier scores with each variable additively. That is, an outer sum would be added to Eq.~\eqref{eq:corrloss} that goes over the new OHE variables encoding the categorical $\pv$.  

\noindent\emph{Multiple $PVs$}.$\;$ \method can handle multiple different $PVs$ simultaneously, such as race and gender, since the loss components Eq.~\eqref{eq:corrloss} and Eq.~\eqref{eq:rankloss} can be used additively for each $\pv$. However, the caveat to additive loss 
is that it would only enforce fairness with respect to each individual $\pv$, and yet may not exhibit fairness for the \textit{joint} distribution of protected variables \cite{kearns2018preventing}. Even when additive extension may not be ideal, we avoid modeling multiple protected variables as a single $\pv$ that induces groups based on values from the cross-product of available values across all $PVs$. This is because partitioning of the data based on cross-product may  yield many small groups, which could cause instability in learning and poor generalization.\looseness=-1



%% file: 040experiment.tex
Our proposed \method is evaluated through extensive experiments on a set of synthetic datasets as well as diverse real-world datasets. In this section, we present dataset description and the experimental setup, followed by key evaluation questions and results.

\subsection{Dataset Description}
Table \ref{tab:datatable} gives an overview of the datasets used in evaluation. 
A brief summary follows, with details on generative process of synthetic data and detailed descriptions in Appendix~\ref{app:syn}.
\label{sub:dataset}

\subsubsection{Synthetic}
We illustrate the efficacy of \method on two synthetic datasets, \synthone and \synthtwo.  These datasets present scenarios that mimic real-world settings, where we may have features that are uncorrelated with the outcome labels but partially correlated with the $\pv$~(see Fig.~\ref{fig:synth}a), or features which are correlated both to outcome labels and $\pv$~(see Fig.~\ref{fig:synth}b). 


\hide{
\begin{figure}[]
	\centering
	\subfloat [\synthone]{\includegraphics[scale=0.36]{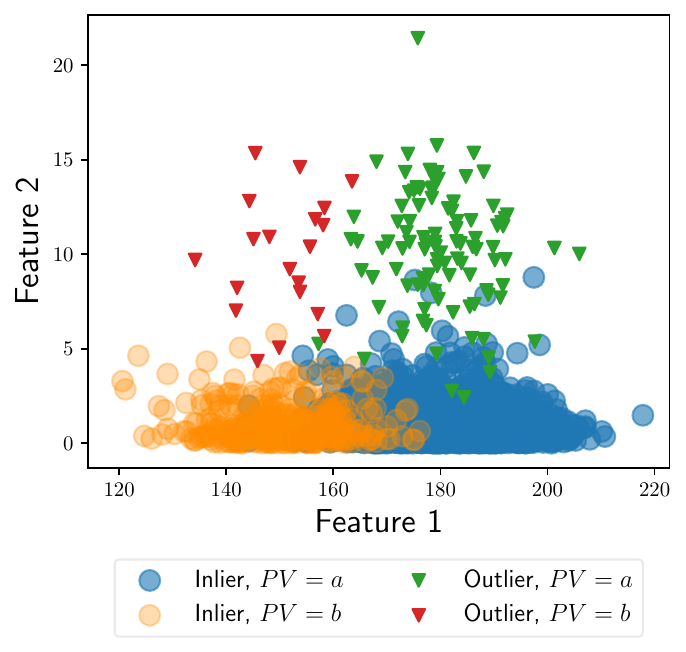}}
	\subfloat[\synthtwo]{\includegraphics[scale=0.36]{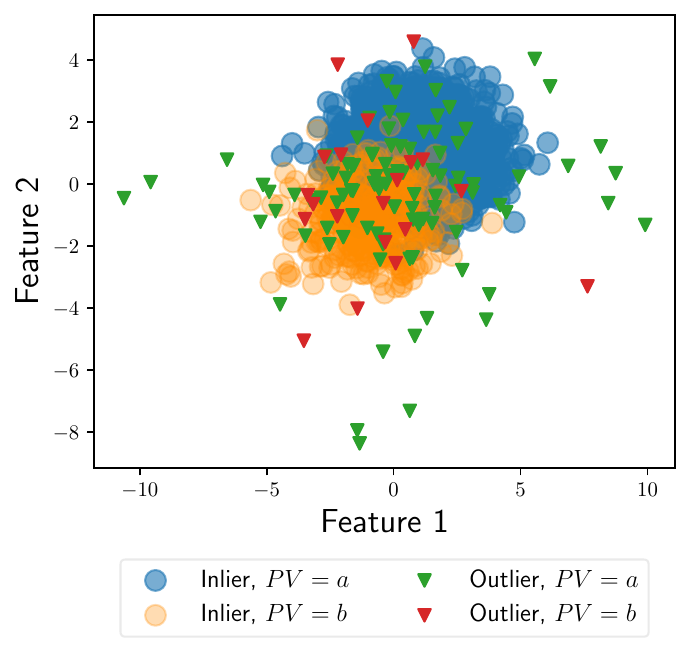}}
\end{figure}
}

\begin{figure}[]
	\centering
	\subfloat [\synthone]{\includegraphics[scale=0.36]{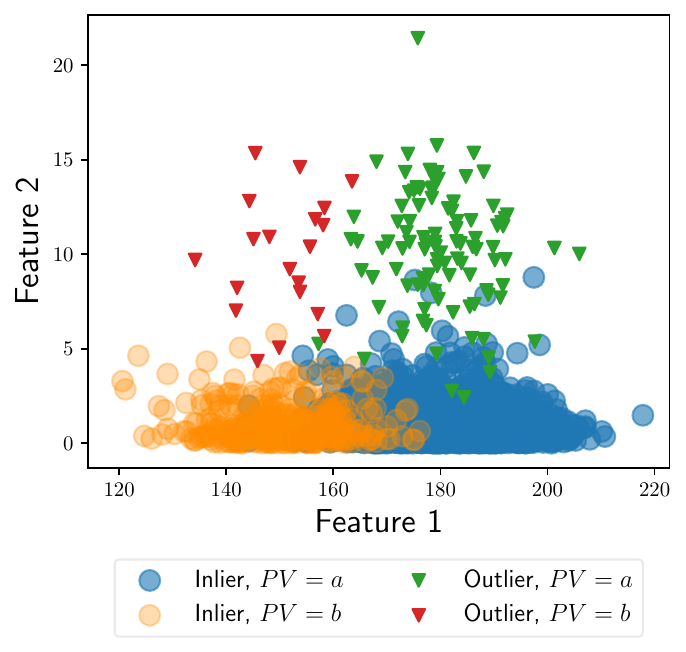}}
	\subfloat[\synthtwo]{\includegraphics[scale=0.36]{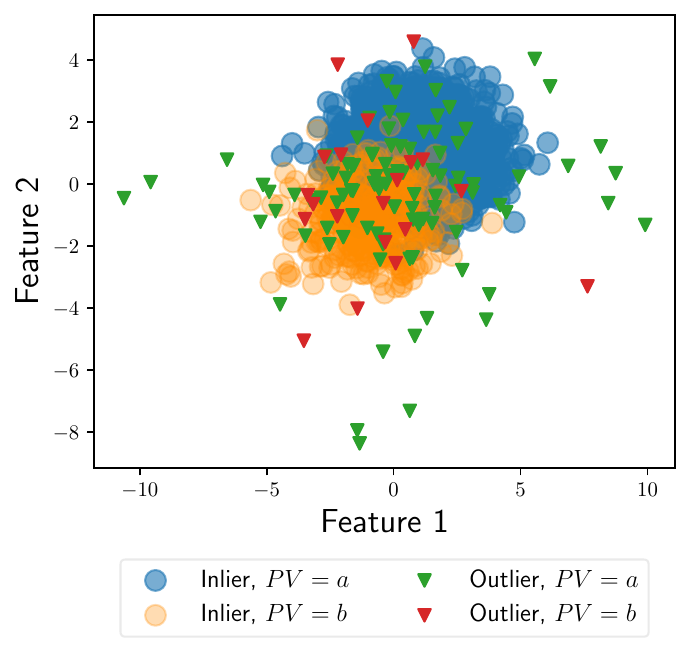}}
	\caption{Synthetic datasets. See Appendix~\ref{app:syn} for the details of the data generating process.\label{fig:synth}}
\end{figure}

\noindent

\subsubsection{Real-world} We experiment on \numdatasets{} real-world datasets from diverse domains that have various types of PV: specifically gender, age, and race (see Table \ref{tab:datatable}). 

\subsection{Baselines} 

We compare \method to two classes of baselines: $(i)$ a fairness-agnostic base detector that aims to solely optimize for detection performance, and $(ii)$ preprocessing methods that aim to correct for bias in the underlying distribution and generate a dataset obfuscating the $\pv$. 

\noindent
\textbf{Base detector model:}
\begin{compactitem}
	\item \base: A deep anomaly detector that employs an autoencoder neural network. The reconstruction error of the autoencoder is used as the anomaly score. \base omits the protected variable from model training.
	
	
	
\end{compactitem}

\noindent
\textbf{Preprocessing based methods:}
\begin{compactitem}
	\item \rw~\cite{kamiran2012data}: A \prem{} approach that assigns weights to observations in each
	group differently to counterbalance the under-representation of minority samples.
	
	\item \disparate~\cite{feldman2015certifying}  A \prem{}  approach that edits feature values such that protected variables can not be predicted based on other features in order to increase group fairness. It uses $repair\_level$ as a hyperparameter, where $0$ indicates no repair, and the larger the value gets, the more obfuscation is enforced. 
	
	\item \lfr: This baseline is based on ~\cite{zemel2013learning} that aims to find a  latent representation of the data while  obfuscating information about protected variables. In our implementation, we omit the classification loss component during representation learning. It uses two hyperparameters -- $A_z$ to control for SP, and $A_x$ to control for the quality of representation. 
	
	\item \crl: This is based on ~\cite{beutel2017data} that finds  new latent representations by employing an adversarial training process to remove information about the protected variables. In our implementation, we use reconstruction error in place of the classification loss. \crl uses $\lambda$ to control for the trade-off between accuracy (in our implementation, reconstruction quality) and obfuscating protected variable.
	This baseline optimizes an objective similar to that proposed in ~\cite{zhang2020towards} which substitutes SVDD loss for reconstruction loss.
\end{compactitem}

The OD task proceeds the preprocessing, where we employ the \base{} detector on the modified data 
transformed or learned by each of the \prem{} based baselines. We do not compare to the LOF-based fair detector in \cite{p2020fairOD} as it exhibits disparate treatment and is inapplicable in settings that we consider.


{\bfseries Hyperparameters} The hyperparameter settings for the competing methods are detailed in Appendix~\ref{app:hyperparam}.


\subsection{Evaluation}
\label{subsec:eval}
We design experiments to answer the following questions:

\begin{compactitem}
    \item {\bf [Q1] Fairness:} How well does \method{} (a) achieve fairness 
    as compared to the baselines, and
    (b) retain the within-group ranking from \base{}? 
    
    \item {\bf [Q2] Fairness-accuracy trade-off:} How accurately are the outliers detected by \method{} as compared to fairness-agnostic \base{} detector?
    
    \item {\bf [Q3] Ablation study:} How do different elements of \method influence group fidelity and detector fairness?
\end{compactitem}





\subsubsection{Evaluation Measures}
\subsubsection*{{\bf\fairness}} Fairness is measured in terms of statistical parity. We use flag-rate ratio $r = \frac{\prob(\flag=1|\pv=a)}{\prob(\flag=1|\pv=b)}$ which measures the statistical fairness of a detector based on the predicted outcome where $\prob(\flag=1|\pv=a)$ is the flag-rate of the \emph{majority} group and $\prob(\flag=1|\pv=b)$ is the flag-rate of the \emph{minority} group. We define \fairness$ = \min(r, {1}/{r}) \in [0,1]$. For a maximally fair detector, \fairness$ = 1$ as $r=1$. 



\subsubsection*{{\bf\gf}} We use the Harmonic Mean~(HM) of per-group \ndcg{} to measure how well the group ranking of \base{} detector is preserved in the fairness-aware detectors. HM between two scalars $p$ and $q$ is defined as $1/(\frac{1}{p} + \frac{1}{q})$. We use HM to report \gf since it is (more) sensitive to lower values (than e.g. arithmetic mean); as such, it takes large values when \textit{both} of its arguments have large values. We define $\text{\gf} = \text{HM} (\ndcg_{PV=a},$ $\ndcg_{PV=b})$, where


\noindent
\resizebox{\linewidth}{!}{
	\begin{minipage}[b]{\linewidth}
		\begin{align*}
		\ndcg_{PV=a} = 	\sum_{i=1}^{|\mathcal{X}_{PV=a}|} \frac{ 2^{s^{\text{\base}}(\xobs_i)} -1 }{\log_2( 1 + \sum_{k=1}^{|\mathcal{X}_{PV=a}|} \mathbbm{1}( \s(\xobs_i) \leq \s(\xobs_k) )) \cdot IDCG } \;,
		\end{align*}
	\end{minipage}
}
 $|\mathcal{X}_{PV=a}|$ is the number of instances in group with $PV=a$, $\mathbbm{1}(cond)$ is the indicator function that evaluates to $1$ if $cond$ is true and $0$ otherwise, $\s(\xobs_i)$ is the predicted score of the fairness-aware detector, 
$s^{\text{\base}}(\xobs_i)$ is the outlier score from \base{} detector and $IDCG = \sum_{j=1}^{|\mathcal{X}_{PV=a}|}\frac{ 2^{s^{\text{\base}}(\xobs_j)} -1 }{\log_2(j+1)}$.
\gf $\approx 1$ indicates that group ranking from the \base detector is well preserved. 

\subsubsection*{{\bf\topkrank}} We also measure how well the final ranking of the method aligns with the purely performance-driven 
\base{} detector, as \base{} optimizes only for reconstruction error. We compute top-$k$ rank agreement as the Jaccard set similarity between the top-$k$ observations as ranked by two methods. Let $\pi^{\text{\base}}_{[1:k]}$ denote the top-$k$ of the ranked list based on outlier scores $s^{\text{\base}}(X_i)$'s, and $\pi^{detector}_{[1:k]}$ be the top-$k$ of the ranked list for competing methods where $detector$$\in$\{\rw, \disparate, \lfr, \crl, \method\}. Then the measure is given as {\small $ \text{\topkrank} = |\pi^{\text{\base}}_{[1:k]} \cap \pi^{detector}_{[1:k]}|/|\pi^{\text{\base}}_{[1:k]} \cup \pi^{detector}_{[1:k]}| $}.

\subsubsection*{{\bf \prratio and \apratio}}  Finally, we consider supervised parity measures based on 
ground-truth labels, defined as the ratio of ROC AUC and Average Precision (AP) performances across groups;
$\text{\prratio} = {\text{AUC}_{PV=a}}/{\text{AUC}_{PV=b}}$ and $\text{\apratio} = {\text{AP}_{PV=a}}/{\text{AP}_{PV=b}}$.



\subsection*{\bfseries [Q1] Fairness}


In Fig.~\ref{fig:combinedrank} (presented in Introduction), \method{} is compared against \base, 
as well as all the \prem baselines across datasets. 
The methods are evaluated using the best configuration of each method\footnote{In Appendix~\ref{app:exp}, for all methods and all datasets, we report detailed values for different metrics for each \pv{} induced group.} on each dataset. The best hyperparameters for \method \ignore{, \corr and \corrscore }are the ones for which \gf and \fairness\footnote{Note that we can do model selection in this manner without access to any labels, since both are unsupervised measures. 
} are closest to the ``ideal'' point as indicated in Fig.~\ref{fig:combinedrank}.

In Fig.~\ref{fig:combinedrank}~(left), the average of  
\fairness and \gf for each method over datasets is reported. \method achieves $9\times$ and $5\times$ improvement in \fairness
as compared to \base method and the nearest competitor, respectively. For \method, \fairness is very close to $1$,
while at the same time the group ranking from the \base detector is well preserved where \gf also approaches $1$. 
\method dominates the baselines (see Fig.~\ref{fig:combinedrank}~(right)) as it is on the Pareto frontier of \gf and \fairness. Here, each point on the plot represents an evaluated dataset. Notice that \method preserves the group ranking while achieving SP consistently across datasets.\looseness=-1 
\begin{figure}[htp]
	\centering
	\includegraphics[width=\linewidth]{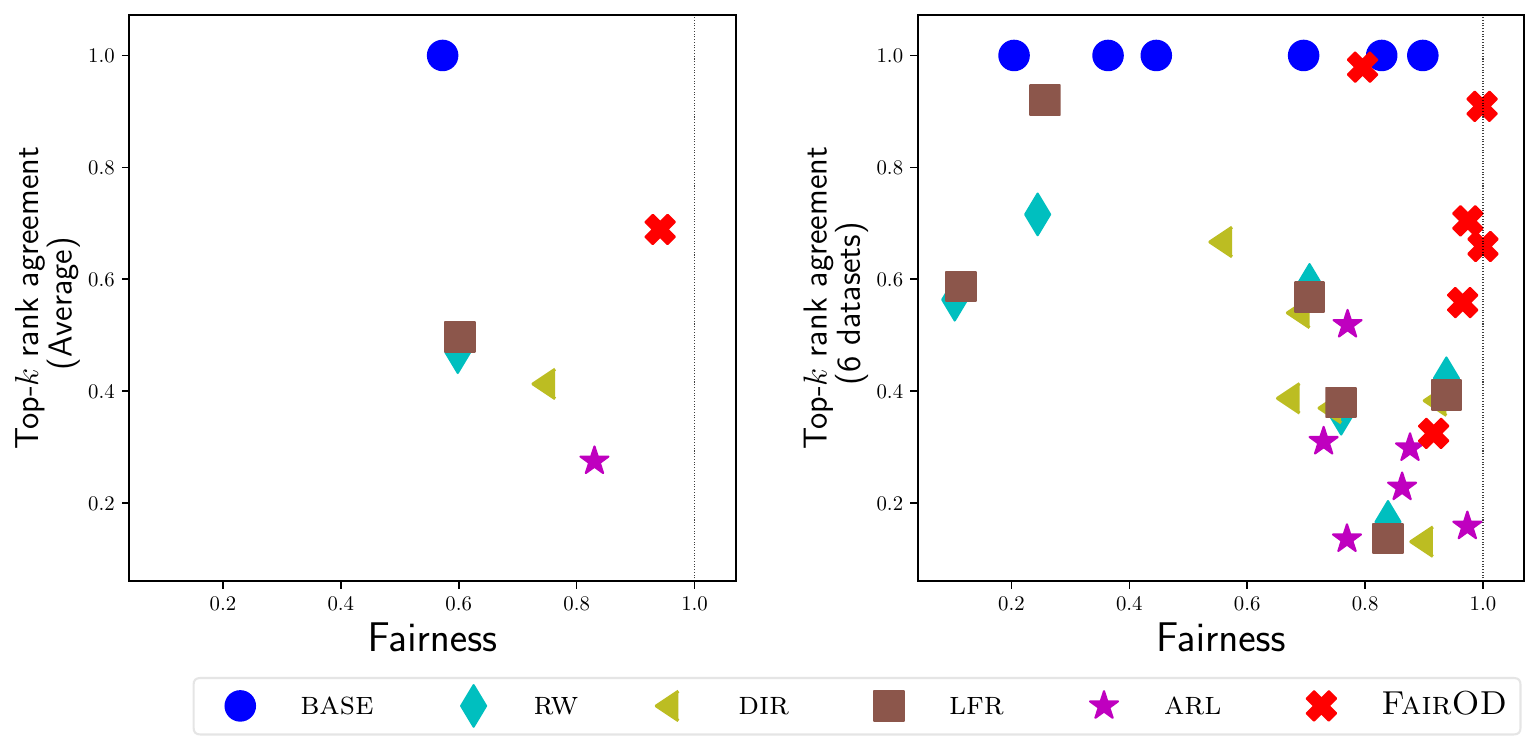}
	\caption{(left) \method achieves the best \topkrank compared to the competitors (\base is shown for reference) in addition to the best overall \fairness, across datasets on average, and
	(right) measures are shown on individual datasets.
	\label{fig:combinedfidelity} 
	}
\end{figure}
Fig.~\ref{fig:combinedfidelity} reports \topkrank (computed at top-5\% of ranked lists) of each method evaluated across datasets. The agreement measures the degree of alignment of the ranked results by a method with the fairness-agnostic \base detector.
In Fig.~\ref{fig:combinedfidelity}~(left), as averaged over datasets, \method achieves better rank agreement as compared to the competitors. In Fig.~\ref{fig:combinedfidelity}~(right), \method approaches ideal statistical parity across datasets while achieving better rank agreement with the \base detector.
Note that \method does not strive for a perfect \topkrank (=1) with \base, since \base is shown to fall short with respect to our desired fairness criteria. Our purpose in illustrating it is to show that the ranked list by \method is not drastically different from \base, which simply aims for detection performance.


Next we evaluate the competing methods against supervised (label-aware) fairness metrics. Note that \method does not (by design) optimize for these supervised fairness measures. 
Fig.~\ref{fig:premethodsAP} evaluates the methods against \fairness and label-aware parity 
criterion  -- specifically, group \apratio (ideal \apratio is $1$). \method 
approaches ideal \fairness as well as ideal \apratio across all datasets. \method outperforms the competitors on the averaged metrics over datasets~(Fig.~\ref{fig:premethodsAP}~(left)) and across individual datasets~(Fig.~\ref{fig:premethodsAP}~(right)). In contrast, the \prem baselines are up to $\sim$$5\times$ worse than \method over \apratio measure across datasets. 
Fig.~\ref{fig:premethodsPR} reports evaluation of methods against \fairness and another label-aware parity measure 
-- specifically, group \prratio (ideal \prratio = $1$). 
As shown in Fig.~\ref{fig:premethodsPR}~(left), \method outperforms all the baselines in expectation as averaged over all datasets. 
Further, in Fig.~\ref{fig:premethodsPR}~(right), \method\ignore{and its variants} consistently approaches ideal \prratio across datasets, while the preprocessing baselines are up to $\sim$$1.9\times$ worse comparatively.\looseness=-1

We note that impressively, \method approaches parity across different supervised fairness measures despite not being able to optimize for label-aware criteria explicitly.



\subsection*{\bfseries [Q2] Fairness-accuracy trade-off}
In the presence of ground-truth outlier labels, the performance of a detector could be measured using a ranking accuracy metric such as area under the ROC curve (ROC AUC).
\begin{figure}[htp]
	\centering
	\subfloat[\fairness vs. \apratio\label{fig:premethodsAP}]{\includegraphics[width=\linewidth]{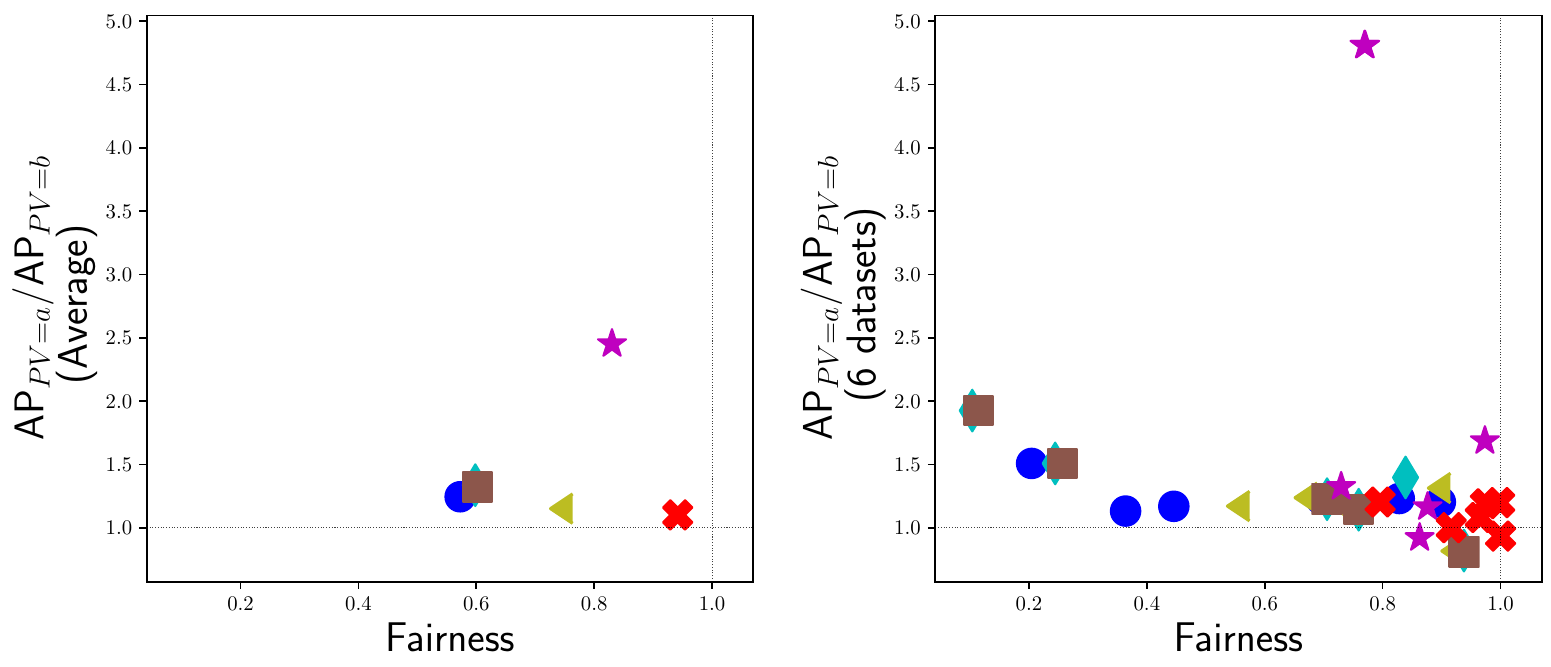}}\\
	\subfloat[\fairness vs. \prratio\label{fig:premethodsPR}]{\includegraphics[width=\linewidth]{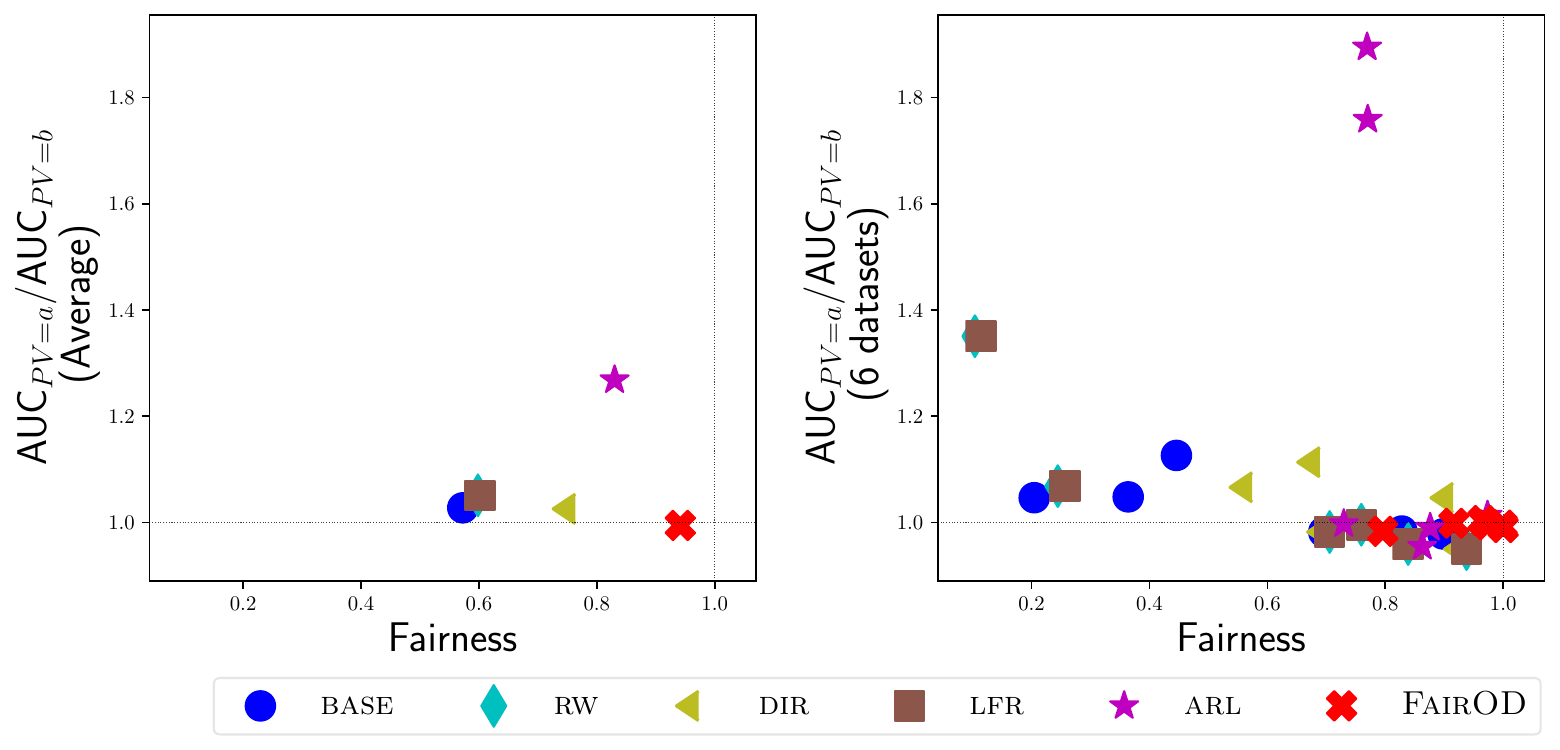}}
	\caption{\method outperforms all competitors on averaged label-aware parity metrics over datasets (left) and for individual datasets (right): we report \fairness against (a) Group \apratio and (b) Group \prratio. \label{fig:premethods}}
\end{figure}

\begin{figure}[htp]
	\centering
	{\includegraphics[width=1.8in]{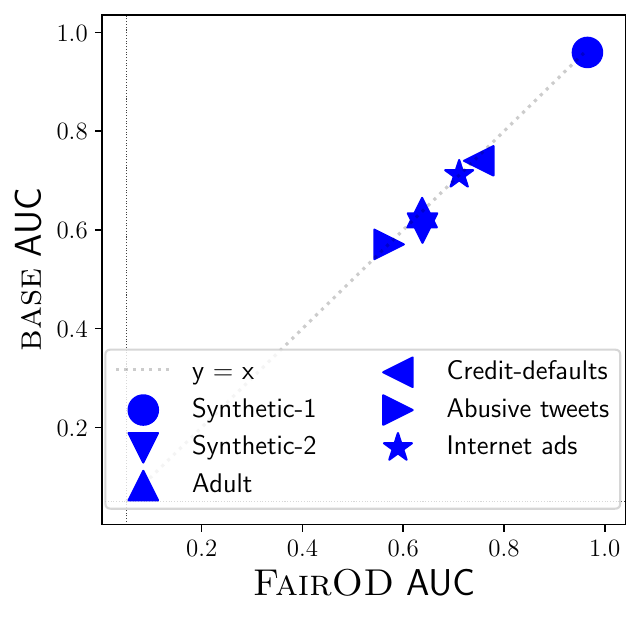}}
	\caption{ROCAUC of \method vs. \base: \method matches the performance of \base detector, while enforcing fairness criteria (maintaining good performance \emph{with} fairness). \label{fig:accuracy}}
\end{figure}

 In Fig.~\ref{fig:accuracy}, we compare the AUC performance of \method to that of \base detector for all datasets. Notice that each of the symbols (i.e. datasets) is slightly below the diagonal line indicating that \method achieves equal or sometimes even better (!) detection performance as compared to \base. The explanation is that since \method enforces SP and does not allow ``laziness",
it addresses the issue of falsely or unjustly flagged minority samples by \base, 
thereby, improving detection performance.

From Fig.~\ref{fig:accuracy}, we conclude that \method 
does not trade-off detection performance much, and in some cases it even improves performance by eliminating false positives from the minority group, as compared to the performance-driven, fairness-agnostic \base.

\subsection*{\bfseries [Q3] Ablation study}
Finally, we evaluate the effect of various components in the design of  \method's objective. Specifically, we compare to the results of two  relaxed variants of \method{}, namely \corr and \corrscore, described as follows.

\begin{itemize}[leftmargin=*]
\item \corr: We retain only the SP-based regularization term from \method objective along with the reconstruction error. This relaxation of \method is partially based on the method proposed in \cite{beutel2019putting}, which minimizes the correlation between model prediction and group membership to the $\pv$. In \corr, the reconstruction error term substitutes the classification loss  used in the optimization criteria in ~\cite{beutel2019putting}. Note that \corr concerns itself with only group fairness to attain SP which may suffer from ``laziness'' (hence, \method-L)~(see Sec.~\ref{sec:prelim}).

\item \corrscore: Instead of training with NDCG-based group fidelity regularization, \corrscore utilizes a simpler regularization, aiming to minimize the correlation (hence, \method-C) of the outlier scores per-group with the corresponding scores from \base{} detector. Thus, \corrscore attempts to maintain group fidelity over the entire ranking within a group, in contrast to \method's NDCG-based regularization which emphasizes the quality of the ranking at the top. Specifically, \corrscore substitutes $\lossF$ in Eq.~\eqref{eq:loss} with the following.
\begin{equation*}
 \resizebox{.95\hsize}{!}{$\lossF = - \sum\limits_{v \in \{a,b\}}  \left\lvert 
	\frac{
		\big(\sum_{\xobs_i \in \mathcal{X}_{\pv=v}} \s(\xobs_i) - \mu_{\s} \big) \; \big(\sum_{\xobs_i \in \mathcal{X}_{\pv=v}} 
		\s^{\text{\base}}(\xobs_i)
		- \mu_{\s^{\text{\base}}} \big)
	}{\sigma_{\s} \; \sigma_{\s^{\text{\base}}}} \right\rvert$}
\end{equation*}


where $v \in \{a, b\}$, and 
$\mu_{\s^{\text{\base}}}$, $\sigma_{\s^{\text{\base}}}$ are defined similar to $\mu_{\s}$, $\sigma_{\s}$ respectively.
\end{itemize}

\begin{figure}[]
	\centering
	\includegraphics[width=\linewidth]{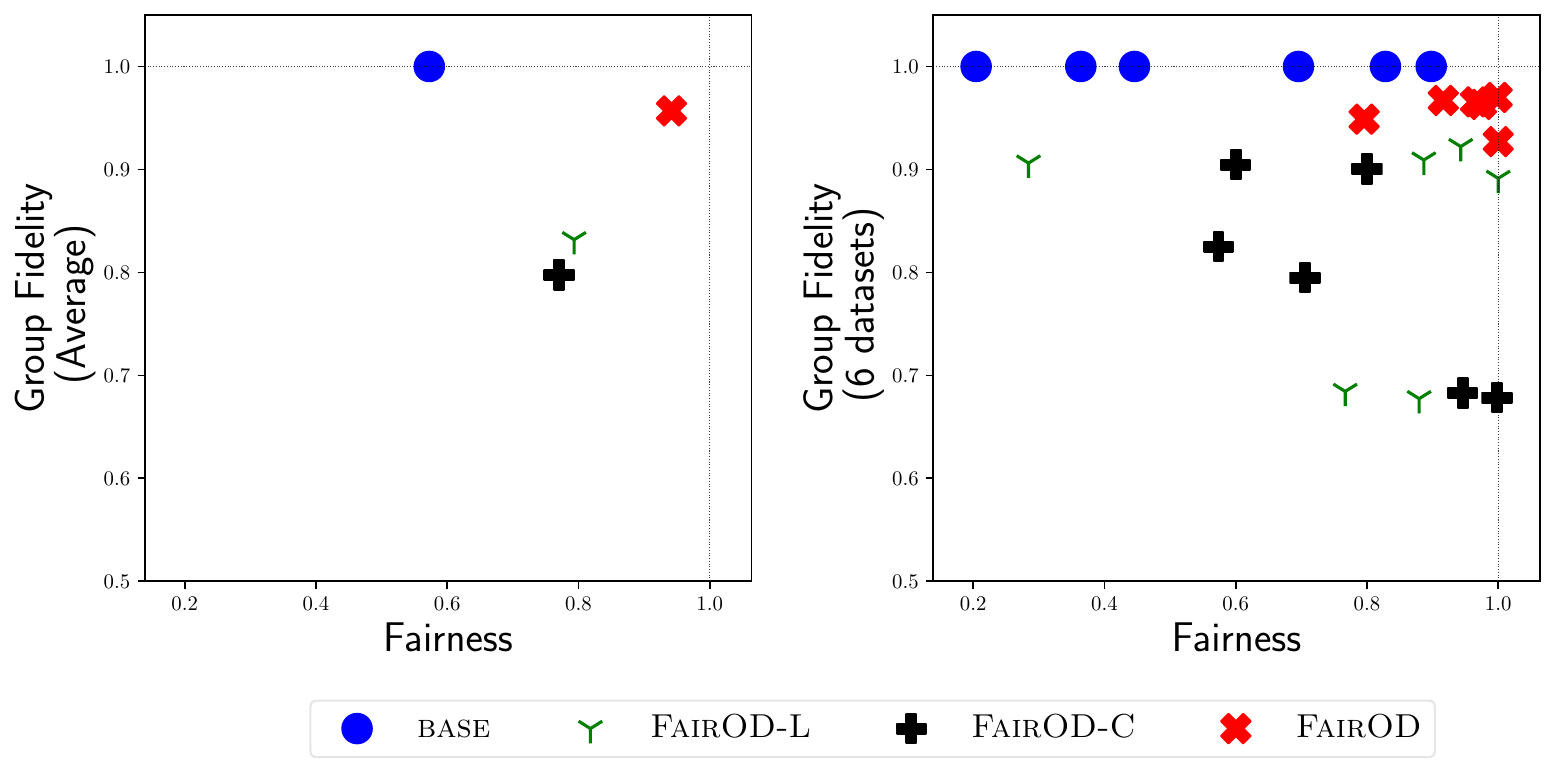}
	\caption{\method compared to its variants \corr and \corrscore across datasets, to evaluate the effect of different regularization components. \corr achieves comparable \fairness to \method while compromising \gf. 
		\corrscore improves \fairness as compared to \base, but is ill-suited to optimizing for \gf.  
		\label{fig:ablation} 
	}
\end{figure}



Fig.~\ref{fig:ablation} presents the comparison of \method and its variants. 
In Fig.~\ref{fig:ablation}~(left), we report the evaluation against \gf and \fairness averaged over datasets, and in Fig.~\ref{fig:ablation}~(right), the metrics are reported for each individual dataset. \corr approaches SP and achieves comparable \fairness to \method except on one dataset as shown in Fig.~\ref{fig:ablation}~(right). This results in lower \fairness compared to \method when averaged over datasets as shown in Fig.~\ref{fig:ablation}~(left). However, \corr suffers with respect to \gf as compared to \method.
This is because \corr may randomly flag instances to achieve SP since it does not include any group ranking criterion in its objective. On the other hand, \corrscore improves \fairness when compared to \base, while under-performing on the majority of datasets compared to \method across metrics. Since \corrscore tries to preserve group-level ranking, it trades-off on \fairness as measured against
\corr. 
We also observe that \method outperforms \corrscore on all datasets, which
suggests that preserving the entire group-level rankings may be a harder task than preserving top of the rankings; it is also a needlessly ill-suited one since what matters for outlier detection is the top of the ranking.

%% file: 050related.tex
A majority of work on algorithmic fairness focuses on supervised learning problems.
We refer to \cite{barocas-hardt-narayanan, mehrabi2019survey} for an excellent overview. We organize related work in three sub-areas related to 
fairness in outlier detection, fairness-aware representation learning, and data de-biasing strategies.

\noindent {\bfseries Outlier Detection and Fairness}
Outlier detection (OD) is a well-studied problem in the literature~\cite{aggarwal2015outlier, gupta2013outlier, chandola2009anomaly}, and finds numerous applications in high-stakes domains like health-care~\cite{luo2010unsupervised}, security~\cite{gogoi2011survey}, and finance~\cite{phua2010comprehensive}. However, only a few studies focus on OD’s fairness aspects. P and Sam Abraham~\cite{p2020fairOD} propose a detector called FairLOF that applies an ad-hoc procedure to introduce fairness specifically to the LOF algorithm \cite{breunig2000lof}. This approach suffers from several drawbacks: (i) it mandates disparate treatment, which may be at times infeasible/unlawful, e.g. in domains like housing or employment, (ii) only prioritizes SP, which as we discussed in Sec. \ref{sec:prelim}, can permit ``laziness,'' (iii) it is heuristic, and cannot be concretely optimized end-to-end. Concurrent to our work, Zhang and Davidson~\cite{zhang2020towards} introduce a deep SVDD based detector employing adversarial training to obfuscate protected group membership, similar to our \crl{} baseline. This approach also has issues: (i) it only considers SP, and (ii) it suffers from well-known instability due to adversarial training~\cite{kodali2017convergence, madras2018learning, cevora2020fair}. A related work by Davidson and 
Ravi \cite{davidson2020framework} focuses on quantifying the fairness of an OD model's outcomes after detection, which thus has a different scope.

\noindent {\bfseries Fairness-aware Representation Learning}
Several works aim to map input samples to an embedding space, where the representations are indistinguishable across groups~\cite{zemel2013learning, louizos2015variational}. Most recently, adversarial training has been used to obfuscate PV association in representations while preserving accurate classification ~\cite{edwards2015censoring, beutel2017data, madras2018learning, adel2019one, zhang2018mitigating}. Most of these methods are supervised.  Substituting classification or label-aware loss terms with unsupervised reconstruction loss can plausibly extend such methods to OD (by using masked representations as inputs to a detector). However, a common shortcoming is that statistical parity (SP) is employed as the primary fairness criterion in these methods, e.g. in fair principal component analysis ~\cite{olfat2019convex} and fair variational autoencoder ~\cite{louizos2015variational}.
To summarize, fair representation learning techniques exhibit two key drawbacks for unsupervised OD: (i) they only employ SP, which may be prone to ``laziness", and (ii) isolating embedding from detection makes embedding oblivious to the task itself, and therefore can yield poor detection performance (as shown in experiments in Sec. \ref{sec:exp}). 

\noindent {\bfseries Strategies for Data De-Biasing}
Some of the popular de-biasing methods~\cite{kamiran2012data, krasanakis2018adaptive} 
draw from topics in learning with imbalanced data~\cite{he2009learning} that employ under- or over-sampling or point-wise weighting of the instances based on the class label proportions to obtain balanced data. These methods apply preprocessing to the data in a manner that is agnostic to the subsequent or downstream task and consider only the fairness notion of SP, which is prone to ``laziness.'' 

%% file: 060conclusion.tex
Although fairness in machine learning has become increasingly prominent in recent years, fairness in the context of unsupervised outlier detection (OD) has received comparatively little study.  OD is an integral data-driven task in a variety of domains including finance, healthcare and security, where it is used to inform and prioritize auditing measures.  Without careful attention, OD as-is can cause unjust flagging of \emph{societal minorities} (w.r.t. race, sex, etc.) because of their standing as \emph{statistical minorities}, when minority status does not indicate positive-class membership (crime, fraud, etc.).  This unjust flagging can propagate to downstream supervised classifiers and further exacerbate the issues.  Our work tackles the problem of fairness-aware outlier detection. Specifically, we first introduce guiding desiderata for, and concrete formalization of the fair OD problem.  We next present \method, a fairness-aware, principled end-to-end detector which addresses the problem, and satisfies several appealing properties: (i) \emph{detection effectiveness:} it is effective, and maintains high detection accuracy, (ii) \emph{treatment parity:} it does not suffer disparate treatment at decision time, (iii) \emph{statistical parity:} it maintains group fairness across minority and majority groups, and (iv) \emph{group fidelity:} it emphasizing flagging of truly high-risk samples within each group, aiming to curb detector ``laziness''. 
Finally, we show empirical results across diverse real and synthetic datasets, demonstrating that our approach achieves fairness goals while providing accurate detection, significantly outperforming unsupervised fair representation learning and data de-biasing based baselines.  
We hope that our expository work yields further studies in this area.


%% file: appendix.tex
\section{Proofs}
\subsection{Proof of Claim~\ref{claim1}}
\label{app:c1}
\begin{proof}
We want OD to exhibit detection effectiveness i.e. $P(Y=1 | O=1) > P(Y=1)$.
\begin{align*}
\text{Now, } \quad P(Y=1 | O=1) = &P(PV=a | O=1) \cdot  \\& P(Y=1 | PV=a, O=1) + \\&P(PV=b| O=1)\cdot  \\& P(Y=1 | PV=b, O=1) \\
\end{align*}

Given SP, we have 
\begin{align*}
&P(O=1|PV=a) = P(O=1|PV=b)\\
\implies & P(PV=a | O=1) = P(PV=a), \text{ and } \\&P(PV=b | O=1) = P(PV=b)
\end{align*}

Therefore, we have 
\begin{align}
\begin{split}
\text{Now, } \quad P(Y=1 | O=1) = &P(PV=a)\cdot \\& P(Y=1 | PV=a, O=1) + \\&P(PV=b) \cdot  \\& P(Y=1 | PV=b, O=1) \\
\end{split}
\end{align}

Now, 
\begin{align*}
P(Y = 1) = &P(PV=a)\cdot P(Y=1 | PV=a) + \\& P(PV=b) \cdot P(Y=1 | PV=b)
\end{align*}

Therefore, if we want $P(Y=1 | O=1) > P(Y=1)$, then
\begin{align}
\label{eq:ineq}
\begin{split}
&P(PV=a)\cdot P(Y=1 | PV=a, O=1) + \\&P(PV=b) \cdot P(Y=1 | PV=b, O=1) \\
&\;\;\;\;\;\;>\\
&P(PV=a)\cdot P(Y=1 | PV=a) +  \\& P(PV=b) \cdot P(Y=1 | PV=b)
\end{split}
\end{align}

\begin{align*}
\implies  \exists v \in \{a, b\} \quad s.t.\; &P(Y=1 | PV=v, O=1) \\
&>  \\&P(Y=1 | PV=v)
\end{align*}
\end{proof}


\subsection{Proof of Claim~\ref{claim2}}
\label{app:c2}
\begin{proof}
Without loss of generality, assume that $P(Y=1 | PV=a, O=1) > P(Y=1 | PV=a)$ i.e. ( i.e. $P(Y=1 | PV=a, O=1) = K \cdot P(Y=1 | PV=a); K > 1 $), and let $\frac{P(Y=1 | PV=a) }{P(Y=1 | PV=b) } = \frac{P(Y=1 | PV=a, O=1) }{P(Y=1 | PV=b, O=1) } = \frac{1}{r}$ then \\
Case 1: When $P(Y=1 | PV=b, O=1) < P(Y=1 | PV=b)$\\

\noindent
\resizebox{\linewidth}{!}{
\begin{minipage}{\linewidth}
	\begin{align*}
	& P(Y=1 | PV=b, O=1) < P(Y=1 | PV=b)\\
	\implies &P(Y=1 | PV=b, O=1) < r \cdot P(Y=1 | PV=a)\\
	\implies &P(Y=1 | PV=b, O=1) < r \cdot P(Y=1 | PV=a, O=1), \\
	&\quad\quad\quad[\because P(Y=1 | PV=a, O=1) > P(Y=1 | PV=a)]
	\end{align*}
\end{minipage}
}

This contradicts our assumption that $ P(Y=1 | PV=b, O=1) = r \cdot P(Y=1 | PV=a, O=1)$, therefore it must be that $P(Y=1 | PV=b, O=1) \boldsymbol{\geq} P(Y=1 | PV=b)$.\\

Case 2: When $P(Y=1 | PV=b, O=1) = P(Y=1 | PV=b)$\\
\noindent
\resizebox{\linewidth}{!}{
	\begin{minipage}{\linewidth}
		\begin{align*}
		& P(Y=1 | PV=b, O=1) = P(Y=1 | PV=b)\\
		\implies &P(Y=1 | PV=b, O=1) = r \cdot P(Y=1 | PV=a)\\
		\implies &P(Y=1 | PV=b, O=1) < r \cdot P(Y=1 | PV=a, O=1), \\
		&\quad\quad\quad[\because P(Y=1 | PV=a, O=1) > P(Y=1 | PV=a)]
		\end{align*}
	\end{minipage}
}

This contradicts our assumption that $ P(Y=1 | PV=b, O=1) = r \cdot P(Y=1 | PV=a, O=1)$, therefore it must be that $P(Y=1 | PV=b, O=1) \boldsymbol{>} P(Y=1 | PV=b)$.\\

Case 3: When $P(Y=1 | PV=b, O=1) > P(Y=1 | PV=b)$ i.e. ($P(Y=1 | PV=b, O=1) = L \cdot P(Y=1 | PV=b); L > 1 $)\\
Now, we know that,
\begin{align*}
&P(Y=1 | PV=a)\cdot P(Y=1 | PV=b, O=1) \\&= P(Y=1 | PV=b) \cdot P(Y=1 | PV=a, O=1)\\
\implies & P(Y=1 | PV=a)\cdot P(Y=1 | PV=b, O=1) \\&= P(Y=1 | PV=b) \cdot K \cdot P(Y=1|PV=a)\\
\implies & P(Y=1 | PV=b, O=1) = K \cdot P(Y=1 | PV=b) \\ 
\implies & P(Y=1 | PV=b, O=1) >  P(Y=1 | PV=b)  
\end{align*}
And, for ratio to be preserved, it must be that $L = K$.

Hence, enforcing preservation of ratios implies base-rates in flagged observations are larger than their counterparts in the population.
\end{proof}

\section{Data Description}

\subsection{Synthetic data}
\label{app:syn}
We illustrate the effectiveness of \method on two synthetic datasets, namely \synthone and \synthtwo~(as illustrated in Fig.~\ref{fig:synth}). 
These datasets 
are constructed to present scenarios that mimic real-world settings, where we may have features which are uncorrelated with respect to outcome labels but partially correlated with $\pv$, or features which are correlated both to outcome labels and $\pv$. 


\begin{itemize}[leftmargin=*]
	\item \synthone: In \synthone, we simulate a 2-dimensional dataset comprised of samples $X = [x_1, x_2]$ where $x_1$ is correlated with the protected variable $\pv$, but does not offer any predictive value with respect to ground-truth outlier labels $\mathcal{\lbl}$, while $x_2$ is correlated with these labels $\mathcal{\lbl}$ (see Fig.~\ref{fig:synth}a).  
	We draw 2400 samples, of which $PV = a$ (majority) for 2000 points, and $PV = b$ (minority) for 400 points.  120 (5\%) of these points are outliers.  $x_1$ differs in terms of shifted means, but equal variances, for both majority and minority groups.  $x_2$ is distributed similarly for both majority and minority groups, drawn from a normal distribution for outliers, and an exponential for inliers.  The detailed generative process for the data is  below, and Fig. \ref{fig:synth}a shows a visual. 
	{\resizebox{0.99\columnwidth}{!}{
			\begin{tabular}{l}
				{\bf \synthone} \\
				\\ 
				Simulate samples $X = [x_1, x_2]$ by... \\
				$	\begin{aligned}[t]
					PV &\sim \text{Bernoulli}(4/5) \\
					Y &\sim \text{Bernoulli}(1/20) \\
					x_1 &\sim
					\begin{cases}
						\text{Normal}(-1, 1.44) \quad\text{if}\quad Y = 0, \;PV = 1 \quad\text{\color{blue} [a, majority; inlier]} \\
						\text{Normal}(1, 1.44) \quad\text{if}\quad Y = 0, \;PV = 0 \quad\text{\color{blue} [b, minority; inlier]}\\
						2\times \text{Exponential}(1) (1 - 2\times \text{Bernoulli}(1/2)) \quad\text{if}\quad Y = 1 \quad\text{\color{blue} [outlier]}
					\end{cases} \\
					x_2 &\sim
					\begin{cases}
						\text{Normal}(-1, 1) \quad\text{if}\quad Y = 0, \;PV = 1 \quad\text{\color{blue} [a, majority; inlier]} \\
						\text{Normal}(1, 1) \quad\text{if}\quad Y = 0, \;PV = 0 \quad\text{\color{blue} [b, minority; inlier]}\\
						2\times \text{Exponential}(1) (1 - 2\times \text{Bernoulli}(1/2)) \quad\text{if}\quad Y = 1 \quad\text{\color{blue} [outlier]}
					\end{cases}
				\end{aligned}$
			\end{tabular}
		}
	}

	\vspace{0.1in}
	\item \synthtwo: In \synthtwo, we again simulate a 2-dimensional dataset comprised of samples $X = [x_1, x_2]$ where $x_1, x_2$ are partially correlated with both the protected variable $\pv$ as well as ground-truth outlier labels $\mathcal{\lbl}$~(see Fig.~\ref{fig:synth}b). 
	We draw 2400 samples, of which $PV = a$ (majority) for 2000 points, and $PV = b$ (minority) for 400 points.  120 (5\%) of these points are outliers.  
	For inliers, both $x_1, x_2$ are normally distributed, and differ across majority and minority groups only in terms of shifted means, but equal variances. Outliers are drawn from a product distribution of an exponential and linearly transformed Bernoulli distribution (product taken for symmetry). The detailed generative process for the data is  below (right), and Fig. \ref{fig:synth}b shows a visual. 
	{\resizebox{0.8\columnwidth}{!}{
			\begin{tabular}{l}
				{\bf \synthtwo} \\
				\\ 
				Simulate samples $X = [x_1, x_2]$ by... \\
				$\begin{aligned}[t]
					PV &\sim \text{Bernoulli}(4/5) \\
					Y &\sim \text{Bernoulli}(1/20) \\
					x_1 &\sim
					\begin{cases}
						\text{Normal}(180, 10) \quad\text{if}\quad PV = 1 \quad\text{\color{blue} [a, majority]} \\
						\text{Normal}(150, 10) \quad\text{if}\quad PV = 0 \quad\text{\color{blue} [b, minority]}
					\end{cases} \\
					x_2 &\sim
					\begin{cases}
						\text{Normal}(10, 3) \quad\text{if}\quad Y = 1 \quad\text{[\color{blue} outlier]} \\
						\text{Exponential}(1) \quad\text{if}\quad Y = 0 \quad\text{\color{blue} [inlier]} \\
					\end{cases}
				\end{aligned}$
			\end{tabular}
		}
	}
	
\end{itemize}

\subsubsection{Real-world data} 
\label{app:real}
We conduct experiments on \numdatasets{} real-world datasets and select them from diverse domains that have different types of (binary) protected variables, specifically gender, age, and race. Detailed descriptions are as follows.

\vspace{0.1in}
\textbullet $\;${\bfseries Adult}~\cite{lichman2013uci} (\adult). The dataset is extracted from the $1994$ Census database where each data point represents a person. The dataset records income level of an individual along with features encoding personal information on education, profession, investment and family. In our experiments, \emph{gender} $\in$ \{\emph{male, female}\} is used as the protected variable where \emph{female} represents minority group and high earning individuals who exceed an annual income of 50,000 i.e. annual \emph{income} $ > 50,000$ are assigned as outliers ($\lbl=1$). We further downsample \emph{female} to achieve a \emph{male} to \emph{female} sample size ratio of 4:1 and ensure that percentage of outliers remains the same (at $5\%$) across groups induced by the protected variable.

\vspace{0.1in}
\textbullet $\;${\bfseries Credit-defaults}~\cite{lichman2013uci} (\credit). This is a risk management dataset from the financial domain that is based on Taiwan's credit card clients' default cases. The data records information of credit card customers including their payment status, demographic factors, credit data, historical bill and payments. Customer \emph{age} is used as the protected variable where \emph{age} $>25$ indicates the majority group and \emph{age} $\leq 25$ indicates the minority group. We assign individuals with delinquent \emph{payment status} as outliers ($\lbl=1$).  The \emph{age} $>25$ to \emph{age} $\leq 25$ imbalance ratio is 4:1 and contains $5\%$ outliers across groups induced by the protected variable.

\begin{table*}[h] 
	\caption{Evaluation measures are reported for the competing methods on the datasets presented in Appendix~\ref{app:syn}.\label{tab:appresults}}
	\subfloat[\synthone \label{tab:synth1}]{
		\centering{\resizebox{0.48\linewidth}{!}{
				\begin{tabular}{rrrrrrrrr} 
					\toprule
					& \multicolumn{2}{c}{\fr} & \multicolumn{2}{c}{\gf} & \multicolumn{2}{c}{\auc} & \multicolumn{2}{c}{\ap} \\ 
					\cmidrule(rr){2-3} \cmidrule(rr){4-5} \cmidrule(rr){6-7} \cmidrule(rr){8-9}
					Method & $\pv=a$          & $\pv=b$        & $\pv=a$     &  $\pv=b$    & $\pv=a$ &  $\pv=b$    & $\pv=a$     & $\pv=b$    \\ 
					\midrule
					\base&0.0262&0.1282&1.0&1.0&0.9594&0.9168&0.8819&0.5849\\
					\rw&0.033&0.135&0.9299&0.9309&0.9794&0.9168&0.8819&0.5849\\
					\disparate&0.0445&0.0775&0.3953&0.9281&0.9742&0.9138&0.8814&0.7529\\
					\lfr&0.0330&0.1350&0.9299&0.9309&0.9794&0.9168&0.8819&0.5849\\
					\crl&0.0520&0.0400&0.9136&0.3955&0.9786&0.5565&0.886&0.1842\\
					\method&0.0500&0.0500&0.9639&0.9671&0.9666&0.9634&0.8166&0.7557\\
					\addlinespace
					\corr&0.0495&0.0525&0.9149&0.9295&0.9017&0.8714&0.599&0.5214\\
					\corrscore&0.0480&0.0600&0.8929&0.9082&0.9499&0.9284&0.7542&0.6501\\
					\bottomrule
	\end{tabular}}}}
	\subfloat[\synthtwo \label{tab:synth2}]{
		\centering{\resizebox{0.48\linewidth}{!}{
				\begin{tabular}{rrrrrrrrr} 
					\toprule
					& \multicolumn{2}{c}{\fr} & \multicolumn{2}{c}{\gf} & \multicolumn{2}{c}{\auc} & \multicolumn{2}{c}{\ap} \\ 
					\cmidrule(rr){2-3} \cmidrule(rr){4-5} \cmidrule(rr){6-7} \cmidrule(rr){8-9}
					Method & $\pv=a$          & $\pv=b$        & $\pv=a$     &  $\pv=b$    & $\pv=a$ &  $\pv=b$    & $\pv=a$     & $\pv=b$    \\ 
					\midrule
					\base&0.0361&0.0811&1.0&1.0&0.6153&0.5464&0.273&0.2335\\
					\rw&0.0205&0.1975&0.9242&0.6313&0.7544&0.5586&0.3973&0.2064\\
					\disparate&0.0465&0.0675&0.4224&0.9164&0.7892&0.7089&0.3921&0.317\\
					\lfr&0.0205&0.1975&0.9242&0.6313&0.7544&0.5586&0.3973&0.2064\\
					\crl&0.0520&0.0400&0.1801&0.1386&0.9786&0.5165&0.886&0.1842\\
					\method&0.0500&0.0500&0.9339&0.9201&0.6357&0.6419&0.2726&0.2918\\
					\addlinespace
					\corr&0.0500&0.0500&0.8984&0.8843&0.6385&0.6472&0.2742&0.2838\\
					\corrscore&0.0450&0.0750&0.8997&0.9095&0.5957&0.5419&0.2665&0.2339\\
					\bottomrule
	\end{tabular}}}}\\
	\subfloat[\adult \label{tab:adult}]{
		\centering{\resizebox{0.48\linewidth}{!}{
				\begin{tabular}{rrrrrrrrr} 
					\toprule
					& \multicolumn{2}{c}{\fr} & \multicolumn{2}{c}{\gf} & \multicolumn{2}{c}{\auc} & \multicolumn{2}{c}{\ap} \\ 
					\cmidrule(rr){2-3} \cmidrule(rr){4-5} \cmidrule(rr){6-7} \cmidrule(rr){8-9}
					Method & $\pv=a$          & $\pv=b$        & $\pv=a$     &  $\pv=b$    & $\pv=a$ &  $\pv=b$    & $\pv=a$     & $\pv=b$    \\ 
					\midrule
					\base&0.0358&0.0433&1.0&1.0&0.6344&0.6449&0.1105&0.0898\\
					\rw&0.0515&0.0391&0.8399&0.8479&0.6323&0.6351&0.1303&0.1141\\
					\disparate&0.0515&0.0391&0.9299&0.9309&0.6323&0.6351&0.1303&0.1141\\
					\lfr&0.0515&0.0391&0.8099&0.8099&0.6323&0.6351&0.1303&0.1141\\
					\crl&0.0507&0.0444&0.9147&0.5765&0.5951&0.6009&0.0987&0.0848\\
					\method&0.0497&0.0511&0.9646&0.9616&0.6374&0.6404&0.1085&0.0912\\
					\addlinespace
					\corr&0.0513&0.0403&0.9178&0.9005&0.6425&0.6312&0.1213&0.1048\\
					\corrscore&0.0527&0.0302&0.8119&0.7877&0.6533&0.6229&0.1872&0.1435\\
					\bottomrule
	\end{tabular}}}}
	\subfloat[\credit \label{tab:credit}]{
		\centering{\resizebox{0.48\linewidth}{!}{
				\begin{tabular}{rrrrrrrrr} 
					\toprule
					& \multicolumn{2}{c}{\fr} & \multicolumn{2}{c}{\gf} & \multicolumn{2}{c}{\auc} & \multicolumn{2}{c}{\ap} \\ 
					\cmidrule(rr){2-3} \cmidrule(rr){4-5} \cmidrule(rr){6-7} \cmidrule(rr){8-9}
					Method & $\pv=a$          & $\pv=b$        & $\pv=a$     &  $\pv=b$    & $\pv=a$ &  $\pv=b$    & $\pv=a$     & $\pv=b$    \\ 
					\midrule
					\base&0.0445&0.064&1.0&1.0&0.7376&0.7512&0.1938&0.1582\\
					\rw&0.0467&0.06627&0.8399&0.8409&0.7376&0.7512&0.1938&0.1582\\
					\disparate&0.0467&0.06627&0.6899&0.6809&0.7376&0.7512&0.1938&0.1582\\
					\lfr&0.0467&0.06627&0.7299&0.7309&0.7376&0.7512&0.1938&0.1582\\
					\crl&0.0471&0.0645&0.5533&0.6118&0.7242&0.7263&0.1396&0.1054\\
					\method&0.0468&0.066&0.9235&0.9421&0.7368&0.7494&0.2134&0.1725\\
					\addlinespace
					\corr&0.0475&0.062&0.7147&0.6564&0.7276&0.7394&0.1246&0.1025\\
					\corrscore&0.0467&0.0662&0.7871&0.8029&0.7327&0.7484&0.1333&0.1091\\
					\bottomrule
		\end{tabular}}}
	}\\
	\subfloat[\tweets\label{tab:tweets}]{
		\centering{\resizebox{0.48\linewidth}{!}{
				\begin{tabular}{rrrrrrrrr} 
					\toprule
					& \multicolumn{2}{c}{\fr} & \multicolumn{2}{c}{\gf} & \multicolumn{2}{c}{\auc} & \multicolumn{2}{c}{\ap} \\ 
					\cmidrule(rr){2-3} \cmidrule(rr){4-5} \cmidrule(rr){6-7} \cmidrule(rr){8-9}
					Method & $\pv=a$          & $\pv=b$        & $\pv=a$     &  $\pv=b$    & $\pv=a$ &  $\pv=b$    & $\pv=a$     & $\pv=b$    \\ 
					\midrule
					\base&0.0369&0.1015&1.0&1.0&0.5739&0.5476&0.061&0.0539\\
					\rw&0.0479&0.0571&0.2882&0.3312&0.5583&0.582&0.0466&0.0334\\
					\disparate&0.0494&0.0507&0.388&0.4178&0.5552&0.5307&0.0454&0.0345\\
					\lfr&0.0479&0.0571&0.4082&0.4422&0.5583&0.582&0.0466&0.0334\\
					\crl&0.0482&0.0558&0.5432&0.5762&0.4912&0.5146&0.0504&0.0442\\
					\method&0.0488&0.0532&0.9668&0.9671&0.569&0.5699&0.0617&0.0617\\
					\addlinespace
					\corr&0.0331&0.1167&0.9137&0.8986&0.5091&0.4237&0.0574&0.0425\\
					\corrscore&0.0501&0.0488&0.6753&0.6903&0.5592&0.5891&0.0627&0.1002\\
					\bottomrule
		\end{tabular}}}
	}
	\subfloat[\ads \label{tab:ads}]{
		\centering{\resizebox{0.48\linewidth}{!}{
				\begin{tabular}{rrrrrrrrr} 
					\toprule
					& \multicolumn{2}{c}{\fr} & \multicolumn{2}{c}{\gf} & \multicolumn{2}{c}{\auc} & \multicolumn{2}{c}{\ap} \\ 
					\cmidrule(rr){2-3} \cmidrule(rr){4-5} \cmidrule(rr){6-7} \cmidrule(rr){8-9}
					Method & $\pv=a$          & $\pv=b$        & $\pv=a$     &  $\pv=b$    & $\pv=a$ &  $\pv=b$    & $\pv=a$     & $\pv=b$    \\ 
					\midrule
					\base&0.0286&0.0318&1.0&1.0&0.7077&0.7234&0.2555&0.2124\\
					\rw&0.0491&0.0523&0.8236&0.7813&0.7286&0.7672&0.4227&0.5183\\
					\disparate&0.0491&0.0523&0.6236&0.5813&0.7286&0.7672&0.4296&0.5253\\
					\lfr&0.0491&0.0523&0.7236&0.6813&0.7286&0.7672&0.4257&0.5253\\
					\crl&0.0499&0.0500&0.5028&0.2181&0.6572&0.6487&0.0885&0.0525\\
					\method&0.0499&0.0500&0.9698&0.9699&0.7179&0.7216&0.2592&0.2163\\
					\addlinespace
					\corr&0.0683&0.0588&0.5551&0.8684&0.7179&0.7345&0.0005&0.0005\\
					\corrscore&0.0499&0.0500&0.6611&0.6966&0.7007&0.7251&0.2636&0.2455\\
					\bottomrule
		\end{tabular}}}
	}
\end{table*}

\vspace{0.1in}
\textbullet $\;${\bfseries Abusive Tweets}~\cite{blodgett2016demographic} (\tweets). The dataset is a collection of Tweets along with annotations indicating whether a tweet is abusive or not. 
The data are not annotated with any protected variable by default; therefore, to assign protected variable to each Tweet, we employ the following process: We predict the racial dialect --- \emph{African-American}  or \emph{Mainstream} --- of the tweets in the corpus using the language model proposed by ~\cite{blodgett2016demographic}. The dialect is assigned to a Tweet only when the prediction probability is greater than 0.7, and then the predicted \emph{racial dialect} is used as protected variable where \emph{African-American} \emph{dialect} represents the minority group. In this setting, abusive tweets are labeled as outliers ($\lbl=1$) for the task of flagging abusive content on Twitter. 
The group sample size ratio of \emph{racial dialect} = \emph{African-American} to \emph{racial dialect} = \emph{Mainstream} is set to 4:1. We further sample data points to ensure equal percentage~($5\%$) of outliers across dialect groups.\looseness=-1

\vspace{0.1in}
\textbullet $\;${\bfseries Internet ads}~\cite{lichman2013uci} (\ads). This is a collection of possible advertisements on web-pages. The features characterize each ad by encoding phrases occurring in the ad URL, anchor text, alt text, and encoding geometry of the ad image. We assign observations with class label \emph{ad} as outliers  ($\lbl=1$)
and downsample the data to get an outlier rate of 5\%. There exists no demographic information available, therefore we simulate a binary protected variable by randomly assigning each observation to one of two values (i.e. groups) $\in \{0, 1\}$ such that the group sample size ratio is 4:1.\looseness=-1

\section{Hyperparameters} 
\label{app:hyperparam}
We choose the hyperparameters of \method from $\alpha \in \{0.01, 0.5, 0.9\} \times \gamma \in \{ 0.01, 0.1, 1.0 \}$ by evaluating the Pareto curve for fairness and group fidelity criteria. The \ignore{\inm}\base and \method methods both use an auto-encoder with two hidden layers. We fix the number of hidden nodes in each layer to $2$ if $d \leq 100$, and $8$ otherwise. 
The representation learning methods \lfr and \crl use the model configurations as proposed by their authors. The hyperparameter grid for the \prem baselines are set as follows: 
$repair\_level \in \{0.0001, 0.001, 0.01, 0.1, 1.0\}$ for \disparate, $A_z\in \{0.0001, 0.001, 0.01, 0.1, 0.9\}$ and $A_x = 1-A_z$ for \lfr, and $\lambda \in \{0.0001, 0.001, 0.01, 0.1, 0.9\}$ for \crl. We pick the best model for the \prem baselines using \fairness 
as they only optimize for statistical parity. The best \base model is selected based on reconstruction error through cross validation upon multiple runs with different random seeds. 


\section{Supplemental Results}
\label{app:exp}
In this section, we report \fr, \gf, \auc and \ap  (see Table~\ref{tab:appresults}) for the competing methods on a set of datasets (see Appendix~\ref{app:real}) w.r.t groups induced by $\pv = v; \;v \in \{a, b\}$ to supplement the experimental results presented in Sec.~\ref{sec:exp}. Notice that in most cases~(see Table~\ref{tab:synth1} through Table~\ref{tab:ads}), \method outperforms the \base model on label-aware parity metrics (\prratio, \apratio) and, furthermore, outperforms \base on at least one of the performance metrics (e.g. \auc, \ap); fairness need not imply worse OD performance.\looseness=-1
